\documentclass{article}

\PassOptionsToPackage{round}{natbib}

\usepackage[preprint]{neurips_2018}

\usepackage[utf8]{inputenc} 
\usepackage[T1]{fontenc}    
\usepackage{hyperref}       
\usepackage{url}            
\usepackage{booktabs}       
\usepackage{amsfonts}       
\usepackage{nicefrac}       
\usepackage{microtype}      

\newcommand{\ourmaintitle}{Testing Conditional Independence on Discrete Data using~Stochastic~Complexity}

\newcommand{\oururl}{\url{https://eda.mmci.uni-saarland.de/sci}}

\usepackage{ntheorem}
\usepackage{pgfplots}
\usepgfplotslibrary{fillbetween}
\usepackage{latexsym}
\usepackage{subcaption}
\usepackage{balance}
\usepackage{amsmath,graphicx,centernot}
\usepackage[ruled, vlined, nofillcomment, linesnumbered]{algorithm2e}
\usepackage{footmisc}

\RequirePackage{amsmath}
\RequirePackage{amssymb}
\RequirePackage[english]{babel}
\RequirePackage{booktabs}
\RequirePackage{color}
\RequirePackage{enumitem} 
\RequirePackage{fancyhdr}
\RequirePackage{graphicx}
\RequirePackage{pgfplots}
\RequirePackage{tikz}
\RequirePackage{url}
\RequirePackage{verbatim}
\RequirePackage{xspace}
\RequirePackage{ifpdf}
\begingroup\expandafter\expandafter\expandafter\endgroup	
\expandafter\ifx\csname IncludeInRelease\endcsname\relax
  \RequirePackage{fixltx2e} 
\fi

\newif\iffinal\finaltrue 
\newif\ifdraft\finalfalse 
\relax

\newcommand{\by}{\@ifstar%
  \BYstar%
  \BYnoStar%
}
\newcommand{\BYstar}[2]{#1\nobreakdash-by\nobreakdash-#2}
\newcommand{\BYnoStar}[2]{\ensuremath{#1\text{\nobreakdash-by\nobreakdash-}#2}}

\ifx \argmax \else
	\DeclareMathOperator*{\argmax}{arg\, max}
\fi

\ifx \argmin \else
	\DeclareMathOperator*{\argmin}{arg\, min}
\fi

\ifx \loglog \else
	\DeclareMathOperator*{\loglog}{log\, log}
\fi

\RequirePackage{mathtools}
\mathtoolsset{showonlyrefs=true}

\providecommand{\note}[1]{{{\color{red} #1}}}
\iffinal
  \ifx\note\undefined
    \newcommand{\note}[1]{{}}%
  \else
    \renewcommand{\note}[1]{{}}%
  \fi
\else
  \ifx\note\undefined
    \newcommand{\note}[2][red]{{{\color{#1} #2}}}%
  \else
    \renewcommand{\note}[2][red]{{{\color{#1} #2}}}%
  \fi
\fi

\definecolor{yafcolor1}{rgb}{0.4, 0.165, 0.553}
\definecolor{yafcolor2}{rgb}{0.949, 0.482, 0.216}
\definecolor{yafcolor3}{rgb}{0.47, 0.549, 0.306}
\definecolor{yafcolor4}{rgb}{0.925, 0.165, 0.224}
\definecolor{yafcolor5}{rgb}{0.141, 0.345, 0.643}
\definecolor{yafcolor6}{rgb}{0.965, 0.933, 0.267}
\definecolor{yafcolor7}{rgb}{0.627, 0.118, 0.165}
\definecolor{yafcolor8}{rgb}{0.878, 0.475, 0.686}

\pgfplotscreateplotcyclelist{yaf}{%
	{yafcolor1},
	{yafcolor2},
	{yafcolor3},
	{yafcolor4},
	{yafcolor5},
	{yafcolor6},
	{yafcolor7},
	{yafcolor8}
}

\pgfplotscreateplotcyclelist{wrongright}{%
	{mambacolor3},
	{mambacolor6}
}

\definecolor{indigo(web)}{rgb}{0.29, 0.0, 0.51}
\definecolor{internationalorange}{rgb}{1.0, 0.31, 0.0}
\definecolor{green(ryb)}{rgb}{0.4, 0.69, 0.2}
\definecolor{richelectricblue}{rgb}{0.03, 0.57, 0.82}
\definecolor{goldenpoppy}{rgb}{0.99, 0.76, 0.0}
\definecolor{crimson}{rgb}{0.86, 0.08, 0.24}
\definecolor{airforceblue}{rgb}{0.36, 0.54, 0.66}

\colorlet{mambacolor1}{indigo(web)}
\colorlet{mambacolor2}{internationalorange}
\colorlet{mambacolor3}{green(ryb)}
\colorlet{mambacolor4}{richelectricblue}
\colorlet{mambacolor5}{goldenpoppy}
\colorlet{mambacolor6}{crimson}
\colorlet{mambacolor7}{airforceblue}

\pgfplotscreateplotcyclelist{mamba}{%
	{mambacolor1},
	{mambacolor2},
	{mambacolor3},
	{mambacolor4},
	{mambacolor5},
	{mambacolor6},
	{mambacolor7}
}

\pgfplotscreateplotcyclelist{mamba-bar}{%
	{indigo(web), fill=indigo(web)},
	{internationalorange, fill=internationalorange},
	{green(ryb), fill=green(ryb)},
	{richelectricblue, fill=richelectricblue},
	{goldenpoppy, fill=goldenpoppy},
	{crimson, fill=crimson},
	{airforceblue, fill=airforceblue}
} 

\pgfplotscreateplotcyclelist{sb-line}{%
	{caribbeangreen},
	{denim},
	{scarlet},
	{green(munsell)},
}

\pgfplotscreateplotcyclelist{sb-bar}{%
	{caribbeangreen, fill=caribbeangreen},
	{denim, fill=denim},
	{scarlet, fill=scarlet},
	{green(munsell), fill=green(munsell)},
}

\pgfplotsset{
	/tikz/normal shift/.code 2 args = {%
		\pgftransformshift{%
			\pgfpointscale{#2}{\pgfplotspointouternormalvectorofticklabelaxis{#1}}%
		}%
	},%
	eda line/.style={
		no markers,
		cycle list name		= mamba,
		tick align        	= outside,
		scaled ticks      	= false,
		enlargelimits     	= false,
		ticklabel shift   	= {10pt},
		axis lines*       	= left,
		line cap          	= round,
		clip              	= false,
		tick style    		= {thin, black, major tick length=2pt},
		x tick label style 	= {font=\scriptsize, yshift = 1pt},
		y tick label style 	= {font=\scriptsize, xshift = 1pt},
		xtick style       	= {normal shift={x}{10pt}},
		ytick style       	= {normal shift={y}{10pt}},
		x axis line style 	= {thick,normal shift={x}{10pt}},
		y axis line style 	= {thick,normal shift={y}{10pt}},
		x label style 		= {at={(axis description cs:0.5,0)}, normal shift={x}{16pt}, anchor=north, font=\scriptsize},
		y label style 		= {at={(axis description cs:0,0.5)}, normal shift={y}{24pt}, anchor=south, font=\scriptsize},
		legend cell align 	= left,
		legend style 		= {inner sep = 1pt, cells = {font=\scriptsize}, },
		legend image code/.code={%
			\draw[mark repeat=2,mark phase=2,#1] 
			plot coordinates { (0cm,0cm) (0.15cm,0cm) (0.3cm,0cm) };%
		}
	}
}
	
\pgfplotsset{
	eda ybar/.style={
		ybar,
		area legend,
		ymajorgrids,
		no markers,
		axis on top,
		xtick				= data,
		cycle list name    	= mamba-bar,
		tick align        	= outside,
		enlargelimits     	= false,
		xmajorgrids 		= false,
		bar width			= 0.7em,
		major grid style	= white,
		axis lines* 		= left,			
		tick style    		= {thin, black, major tick length=2pt},
		major y tick style	= {draw=none},
		x tick label style 	= {font=\scriptsize, yshift=1pt},
		y tick label style = {font=\scriptsize, xshift=1pt},
		x axis line style 	= {thick, normal shift={x}{0pt}},
		y axis line style	= {opacity=0},	
		x label style 		= {at={(axis description cs:0.5,0)}, normal shift={x}{6pt}, anchor=north, font=\scriptsize},
		y label style 		= {at={(axis description cs:0,0.5)}, normal shift={y}{14pt}, anchor=south, font=\scriptsize},
		legend image post style={scale=0.25},
		legend style 		= {inner sep=1pt, cells={font=\scriptsize}, },
		legend cell align 	= left
	}
}

\pgfplotsset{
	eda ybar log/.style={
		eda ybar,
		area legend,
		ymajorgrids,
		no markers,
		axis on top,
		xtick				= data,
		cycle list name    	= mamba-bar,
		tick align        	= outside,
		enlargelimits     	= false,
		xmajorgrids 		= false,
		bar width			= 0.7em,
		major grid style	= white,
		axis lines* 		= left,			
		tick style    		= {thin, black, major tick length=2pt},
		minor y tick style	= {draw=none},
		major y tick style	= {draw=none},
		x tick label style 	= {font=\scriptsize, yshift=1pt},
		y tick label style = {font=\scriptsize, xshift=1pt},
		x axis line style 	= {thick, normal shift={x}{0pt}},
		y axis line style	= {opacity=0},	
		x label style 		= {at={(axis description cs:0.5,0)}, normal shift={x}{10pt}, anchor=north, font=\scriptsize},
		y label style 		= {at={(axis description cs:0,0.5)}, normal shift={y}{4pt}, anchor=south, font=\scriptsize},
		legend image post style={scale=0.25},
		legend style 		= {inner sep=1pt, cells={font=\scriptsize}, },
		legend cell align 	= left
	}
}

\pgfplotsset {
	eda scatter3/.style={
		only marks,
		mark				= *,
		cycle list name		= mamba,
		scaled ticks      	= false,
		enlargelimits     	= false,
		axis lines*			= left,
		mark size			= 0.75pt,
		tick pos 			= left,
		tick align			= outside,
		ticklabel shift   	= {10pt},
		clip              	= false,
		tick style    		= {thin, black, major tick length=2pt},
		xtick style       	= {normal shift={x}{10pt}},
		ytick style       	= {normal shift={y}{10pt}},
		x tick label style 	= {font=\scriptsize, yshift = 1pt},
		y tick label style 	= {font=\scriptsize, xshift = 1pt},
		x axis line style 	= {thick, normal shift={x}{10pt}},
		y axis line style 	= {thick,normal shift={y}{10pt}},
		x label style 		= {at={(axis description cs:0.5,0)}, normal shift={x}{12pt}, anchor=north, font=\scriptsize},
		y label style 		= {at={(axis description cs:0,0.5)}, normal shift={y}{16pt}, anchor=south, font=\scriptsize},
		scatter/use mapped color={draw=indigo(web),fill=indigo(web)},
		legend cell align 	= left,
		legend style 		= {inner sep = 1pt, cells = {font=\scriptsize}, },
		legend image code/.code={%
			\draw[mark repeat=2,mark phase=2,#1] 
			plot coordinates { (0cm,0cm) (0.15cm,0cm) (0.3cm,0cm) };%
		}
	}
}

\pgfplotsset {
	eda scatter4/.style={
		mark				= *,
		cycle list name		= mamba,
		scaled ticks      	= false,
		enlargelimits     	= false,
		axis lines*			= left,
		mark size			= 0.75pt,
		tick pos 			= left,
		tick align			= outside,
		ticklabel shift   	= {10pt},
		clip              	= false,
		tick style    		= {thin, black, major tick length=2pt},
		xtick style       	= {normal shift={x}{10pt}},
		ytick style       	= {normal shift={y}{10pt}},
		x tick label style 	= {font=\scriptsize, yshift = 1pt},
		y tick label style 	= {font=\scriptsize, xshift = 1pt},
		x axis line style 	= {thick, normal shift={x}{10pt}},
		y axis line style 	= {thick,normal shift={y}{10pt}},
		x label style 		= {at={(axis description cs:0.5,0)}, normal shift={x}{12pt}, anchor=north, font=\scriptsize},
		y label style 		= {at={(axis description cs:0,0.5)}, normal shift={y}{16pt}, anchor=south, font=\scriptsize},
		scatter/use mapped color={draw=indigo(web),fill=indigo(web)},
		legend cell align 	= left,
		legend style 		= {inner sep = 1pt, cells = {font=\scriptsize}, },
		legend image code/.code={%
			\draw[mark repeat=2,mark phase=2,#1] 
			plot coordinates { (0cm,0cm) (0.15cm,0cm) (0.3cm,0cm) };%
		}
	}
}

\pgfplotsset{
    box plot/.style={
        /pgfplots/.cd,
        black,
        only marks,
        mark=-,
        tick style    		= {thin, black, major tick length=2pt},
	major y tick style	= {draw=none},
	x tick label style 	= {font=\scriptsize, yshift=1pt},
	y tick label style = {font=\scriptsize, xshift=-2pt},
	x axis line style 	= {thick, normal shift={x}{0pt}},
	y axis line style	= {opacity=0},	
	x label style 		= {at={(axis description cs:0.5,0)}, normal shift={x}{6pt}, anchor=north, font=\scriptsize},
	y label style 		= {at={(axis description cs:0,0.5)}, normal shift={y}{20pt}, anchor=south, font=\scriptsize},
	xtick pos = left,
	ytick pos = left,
        mark size=\pgfkeysvalueof{/pgfplots/box plot width},
        /pgfplots/error bars/y dir=plus,
        /pgfplots/error bars/y explicit,
        /pgfplots/table/x index=\pgfkeysvalueof{/pgfplots/box plot x index},
    },
    box plot box/.style={
        /pgfplots/error bars/draw error bar/.code 2 args={%
            \draw  ##1 -- ++(\pgfkeysvalueof{/pgfplots/box plot width},0pt) |- ##2 -- ++(-\pgfkeysvalueof{/pgfplots/box plot width},0pt) |- ##1 -- cycle;
        },
        /pgfplots/table/.cd,
        y index=\pgfkeysvalueof{/pgfplots/box plot box top index},
        y error expr={
            \thisrowno{\pgfkeysvalueof{/pgfplots/box plot box bottom index}}
            - \thisrowno{\pgfkeysvalueof{/pgfplots/box plot box top index}}
        },
        /pgfplots/box plot
    },
    box plot top whisker/.style={
        /pgfplots/error bars/draw error bar/.code 2 args={%
            \pgfkeysgetvalue{/pgfplots/error bars/error mark}%
            {\pgfplotserrorbarsmark}%
            \pgfkeysgetvalue{/pgfplots/error bars/error mark options}%
            {\pgfplotserrorbarsmarkopts}%
            \path ##1 -- ##2;
        },
        /pgfplots/table/.cd,
        y index=\pgfkeysvalueof{/pgfplots/box plot whisker top index},
        y error expr={
            \thisrowno{\pgfkeysvalueof{/pgfplots/box plot box top index}}
            - \thisrowno{\pgfkeysvalueof{/pgfplots/box plot whisker top index}}
        },
        /pgfplots/box plot
    },
    box plot bottom whisker/.style={
        /pgfplots/error bars/draw error bar/.code 2 args={%
            \pgfkeysgetvalue{/pgfplots/error bars/error mark}%
            {\pgfplotserrorbarsmark}%
            \pgfkeysgetvalue{/pgfplots/error bars/error mark options}%
            {\pgfplotserrorbarsmarkopts}%
            \path ##1 -- ##2;
        },
        /pgfplots/table/.cd,
        y index=\pgfkeysvalueof{/pgfplots/box plot whisker bottom index},
        y error expr={
            \thisrowno{\pgfkeysvalueof{/pgfplots/box plot box bottom index}}
            - \thisrowno{\pgfkeysvalueof{/pgfplots/box plot whisker bottom index}}
        },
        /pgfplots/box plot
    },
    box plot median/.style={
        /pgfplots/box plot,
        /pgfplots/table/y index=\pgfkeysvalueof{/pgfplots/box plot median index}
    },
    box plot width/.initial=1em,
    box plot x index/.initial=0,
    box plot median index/.initial=1,
    box plot box top index/.initial=2,
    box plot box bottom index/.initial=3,
    box plot whisker top index/.initial=4,
    box plot whisker bottom index/.initial=5,
}

\pgfplotsset{
	eda surf/.style={
		view={56}{26},
		axis lines=left,		
		tick style    		= {thin, black, major tick length=2pt},
		 xmajorgrids, x dir= reverse, ymajorgrids, zmajorgrids,
		minor y tick style	= {draw=none},
		major y tick style	= {draw=none},
		major z tick style	= {draw=none},
		x tick label style 	= {font=\scriptsize, yshift=1pt},
		y tick label style 	= {font=\scriptsize, xshift=-3pt, yshift=3pt},
		z tick label style 	= {font=\scriptsize, xshift=1pt},
		z axis line style 		= {thick, normal shift={x}{0pt}},
		x axis line style		= {opacity=0},	
		y axis line style		= {opacity=0},	
		z label style 		= {at={(axis description cs:-0.15,0.5)}, normal shift={z}{10pt}, anchor=south, font=\scriptsize},
		x label style 		= {at={(axis description cs:-0.15,-0.15)}, anchor=south, rotate=-35, font=\scriptsize},
		y label style 		= {at={(axis description cs:0.85,-0.2)}, rotate=15, anchor=south, font=\scriptsize},
		legend image post style={scale=0.25},
		legend style 		= {inner sep=1pt, cells={font=\scriptsize}, },
		legend cell align 	= left,
		grid=major,
		colormap={reverse hot}{
        			indices of colormap={
	            		\pgfplotscolormaplastindexof{hot},...,0 of hot}
    		}
	}
}

\pgfplotsset{
	eda surf2/.style={
		view={56}{26},
		axis lines=left,		
		tick style    		= {thin, black, major tick length=2pt},
		 xmajorgrids, x dir= reverse, ymajorgrids, zmajorgrids,
		minor y tick style	= {draw=none},
		major y tick style	= {draw=none},
		major z tick style	= {draw=none},
		x tick label style 	= {font=\scriptsize, yshift=1pt},
		y tick label style 	= {font=\scriptsize, xshift=-3pt, yshift=3pt},
		z tick label style 	= {font=\scriptsize, xshift=1pt},
		z axis line style 		= {thick, normal shift={x}{0pt}},
		x axis line style		= {opacity=0},	
		y axis line style		= {opacity=0},	
		z label style 		= {at={(axis description cs:-0.15,0.5)}, normal shift={z}{10pt}, anchor=south, font=\scriptsize},
		x label style 		= {at={(axis description cs:-0.15,-0.15)}, anchor=south, rotate=-35, font=\scriptsize},
		y label style 		= {at={(axis description cs:0.85,-0.2)}, rotate=15, anchor=south, font=\scriptsize},
		legend image post style={scale=0.25},
		legend style 		= {inner sep=1pt, cells={font=\scriptsize}, },
		legend cell align 	= left,
		grid=major,
	}
}

\graphicspath{ {./} }

\ifpdf
\usetikzlibrary{external}
\fi

\title{\ourmaintitle}

\SetKwComment{tcpas}{\{}{\}}
\SetCommentSty{textnormal}
\SetArgSty{textnormal}
\SetKw{False}{false}
\SetKw{True}{true}
\SetKw{Null}{null}
\SetKwInOut{Output}{output}
\SetKwInOut{Input}{input}
\SetKw{AND}{and}
\SetKw{OR}{or}
\SetKw{Continue}{continue}

\newcommand{\SC}{\mathit{S}\xspace}
\newcommand{\SCI}{\ensuremath{\mathit{SCI}}\xspace}
\newcommand{\JIC}{\ensuremath{\mathit{JIC}}\xspace}
\newcommand{\CMI}{\ensuremath{\mathit{CMI}}\xspace}
\newcommand{\Penalty}{\ensuremath{\mathcal{R}}\xspace}

\newcommand{\pcmb}{\textsc{PCMB}\xspace}

\newcommand{\pa}{\mathit{PA}\xspace}

\newcommand{\ci}{\mathit{ci}\xspace}

\newcommand{\nml}{\mathit{NML}\xspace}

\newcommand{\fallingfactorial}[1]{%
  ^{\mspace{2mu}\underline{\mspace{-2mu}#1\mspace{-2mu}}\mspace{2mu}}%
}
\newcommand{\risingfactorial}[1]{%
  ^{\mspace{2mu}\overline{\mspace{-2mu}#1\mspace{-2mu}}\mspace{2mu}}%
}

\makeatletter
\renewtheoremstyle{plain}
  {\item{\theorem@headerfont ##1\ ##2\theorem@separator}~}
  {\item{\theorem@headerfont ##1\ ##2\ (##3)\theorem@separator}~}
\makeatother

{\theoremheaderfont{\upshape\bfseries}
 \theorembodyfont{\normalfont\slshape}
 \newtheorem{theorem}{Theorem}}
{\theoremheaderfont{\upshape\bfseries}
 \theorembodyfont{\normalfont\slshape}
 \newtheorem{definition}{Definition}}
 {\theoremheaderfont{\upshape\bfseries}
 \theorembodyfont{\normalfont\slshape}
 }
 {\theoremheaderfont{\upshape\bfseries}
 \theorembodyfont{\normalfont\slshape}
 }
 {\theoremheaderfont{\upshape\bfseries}
 \theorembodyfont{\normalfont\slshape}
 \newtheorem{lemma}{Lemma}}
 {\theoremheaderfont{\upshape\bfseries}
 \theorembodyfont{\normalfont\slshape}
 \newtheorem{example}{Example}}

\newenvironment{proof}{\paragraph{Proof:}}{\hfill$\square$}

\renewcommand{\models}{\mathcal{M}\xspace}
\newcommand{\regret}{\mathcal{C}\xspace}

\newcommand\myeq{\mkern1.5mu{=}\mkern1.5mu}

\newcommand\independent{\protect\mathpalette{\protect\independenT}{\perp}}
\def\independenT#1#2{\mathrel{\rlap{$#1#2$}\mkern2mu{#1#2}}}

	\tikzstyle{flatlabel}  = [above, font = \tiny, inner sep = 1pt, text = black]
	\tikzstyle{flatlabelb}  = [below, font = \tiny, inner sep = 1pt, text = black]
	\tikzstyle{slopelabel}  = [sloped, above, font = \tiny, inner sep = 1pt, text = black]
	\tikzstyle{slopelabelb}  = [sloped, below, font = \tiny, inner sep = 1pt, text = black]

\usepackage{tikz}
\usepackage{pgfplots}
\usepackage{pgfplotstable}
\usetikzlibrary{arrows,shapes,positioning,fit,calc,matrix,trees}
\usetikzlibrary{external}

\pgfdeclareplotmark{btriangle}{%
  \pgfpathmoveto{\pgfpoint{-\pgfplotmarksize}{\pgfplotmarksize}}
  \pgfpathlineto{\pgfpoint{0pt}{-\pgfplotmarksize}}
  \pgfpathlineto{\pgfpoint{\pgfplotmarksize}{\pgfplotmarksize}}
  \pgfpathlineto{\pgfpoint{-\pgfplotmarksize}{\pgfplotmarksize}}
  \pgfusepathqstroke
}
\pgfdeclareplotmark{btriangle*}{%
  \pgfpathmoveto{\pgfqpoint{0pt}{-\pgfplotmarksize}}
  \pgfpathlineto{\pgfqpointpolar{150}{\pgfplotmarksize}}
  \pgfpathlineto{\pgfqpointpolar{30}{\pgfplotmarksize}}
  \pgfpathclose
  \pgfusepathqfillstroke
}

\pgfdeclareplotmark{ltriangle}{%
  \pgfpathmoveto{\pgfqpoint{\pgfplotmarksize}{0pt}}
  \pgfpathlineto{\pgfqpointpolar{-120}{\pgfplotmarksize}}
  \pgfpathlineto{\pgfqpointpolar{120}{\pgfplotmarksize}}
  \pgfpathclose
  \pgfusepathqstroke
}
\pgfdeclareplotmark{ltriangle*}{%
  \pgfpathmoveto{\pgfqpoint{\pgfplotmarksize}{0pt}}
  \pgfpathlineto{\pgfqpointpolar{-120}{\pgfplotmarksize}}
  \pgfpathlineto{\pgfqpointpolar{120}{\pgfplotmarksize}}
  \pgfpathclose
  \pgfusepathqfillstroke
}

\pgfdeclareplotmark{rtriangle}{%
  \pgfpathmoveto{\pgfqpoint{-\pgfplotmarksize}{0pt}}
  \pgfpathlineto{\pgfqpointpolar{-60}{\pgfplotmarksize}}
  \pgfpathlineto{\pgfqpointpolar{60}{\pgfplotmarksize}}
  \pgfpathclose
  \pgfusepathqstroke
}

\definecolor{yafaxiscolor}{rgb}{0.3, 0.3, 0.3}
\definecolor{yafcolor1}{rgb}{0.4, 0.165, 0.553}
\definecolor{yafcolor2}{rgb}{0.949, 0.482, 0.216}
\definecolor{yafcolor3}{rgb}{0.47, 0.549, 0.306}
\definecolor{yafcolor4}{rgb}{0.925, 0.165, 0.224}
\definecolor{yafcolor5}{rgb}{0.141, 0.345, 0.643}
\definecolor{yafcolor6}{rgb}{0.965, 0.633, 0.267}
\definecolor{yafcolor7}{rgb}{0.627, 0.118, 0.165}
\definecolor{yafcolor8}{rgb}{0.878, 0.475, 0.686}

\pgfplotscreateplotcyclelist{yaf}{%
{yafcolor1,mark options={scale=0.45},mark=*},
{yafcolor5,mark options={scale=0.60},mark=triangle*},
{yafcolor3,mark options={scale=0.45},mark=square*},
{yafcolor4,mark options={scale=0.70},mark=btriangle*},
{yafcolor8,mark options={scale=0.60},mark=ltriangle*},
{yafcolor7,mark options={scale=0.60},mark=rtriangle*},
{yafcolor2,mark options={scale=0.60},mark=diamond*},
{yafcolor6,mark options={scale=0.60},mark=pentagon*}}

\tikzset{
precise pin/.style args={[#1][#2]#3:#4}{
    pin={[inner sep=0pt, #1, label={[append after command={
		node [#2,
			outer sep = 0pt,
			inner sep=0pt,
			at=(\tikzlastnode),
			anchor=#3+180 ] {#4} } ]center:{}}]#3:{}}
}}

\pgfplotsset{
	clip = false,
	clip marker paths = true,
	tick align=outside,
	x tick label style = {font=\scriptsize, yshift = 1pt},
	y tick label style = {font=\scriptsize, xshift = 1pt},
	major tick length = 2pt,
    every axis y label/.style = {at = {(ticklabel cs:0.5)}, rotate=90, anchor=center, font=\scriptsize, xshift = 2pt},
	every axis x label/.style = {at = {(ticklabel cs:0.5)}, anchor=center, font=\scriptsize, yshift = -2pt},
	axis y line*=left, axis x line*=bottom,
        enlargelimits = 0.03
}
\tikzstyle{every pin}=[font=\footnotesize, inner sep = 0pt, distance=2em]
\tikzstyle{every pin edge}=[line width = 0.1pt, pin distance = 2em]

\newlength{\myheight}
\newlength{\mywidth}

\newcommand{\legDist}{Distance ($\savg{\distc{}}$)}
\newcommand{\legJacc}{Jacc.\ dist.\ ($\savg{\jacc{}}$)}
\newcommand{\legCover}{Coverage ($\cover{}$)}
\newcommand{\legDens}{Density ($\savg{\density{}}$)}

\colorlet{graphcl1}{yafcolor1!50}
\colorlet{graphcl2}{yafcolor4!30}
\colorlet{graphcl3}{yafcolor2!50}
\colorlet{graphcl4}{yafcolor5}
\colorlet{graphcl5}{yafcolor4}
\colorlet{graphcl6}{yafcolor6}

\tikzstyle{graphedge} = [black, thick, opacity = 0.5]
\tikzstyle{graphnode} = [draw = black, circle, line width = 0pt, text = black, inner sep = 0.5pt, text width = 10pt, align = center]
\tikzstyle{outliernode} = [circle, line width = 0pt, draw, text = black, fill = white, inner sep = 0.5pt, text width = 10pt, align = center]

\tikzstyle{toyedge} = [->, black, thick, bend left = 10, yafcolor5]
\tikzstyle{toynode} = [draw = black, thick, circle, line width = 0pt, text = black, inner sep = 0pt, text width = 13pt, align = center]
\tikzstyle{groupline} = [black, thick, dashed]

\makeatletter
\tikzset{multicircle/.style  args={#1, (#2)}{%
 alias=tmp@name, %
  postaction={%
    insert path={
     \pgfextra{%
     \pgfpointdiff{\pgfpointanchor{\pgf@node@name}{center}}%
                  {\pgfpointanchor{\pgf@node@name}{east}}%
     \pgfmathsetmacro\insiderad{\pgf@x}%
     \foreach \c [count=\ci from = 0, evaluate=\ci as \angle using 360 - (\ci) * #1] in {#2}%
        \fill[\c] (\pgf@node@name.center)  -- ++(0:\insiderad-\pgflinewidth) arc (0:\angle:\insiderad-\pgflinewidth)--cycle;%
        }}}}}
\makeatother

\pgfplotsset{
  boxplot/box width/.initial=1em,
  solid boxes/.style={
    mark=x,
    boxplot/draw direction=y,
    boxplot/whisker extend=0,
    boxplot/draw/median/.code={%
      \draw[mark size=2pt,/pgfplots/boxplot/every median/.try]
        \pgfextra
        \pgftransformshift{
          \pgfplotsboxplotpointabbox
            {\pgfplotsboxplotvalue{median}}
            {0.5}
        }
        \pgfsetfillcolor{white}
        \pgfuseplotmark{*}
        \endpgfextra
      ;
    },
    boxplot/draw/box/.code={
      \draw[fill,/pgfplots/boxplot/every box/.try]
        ($(boxplot box cs:\pgfplotsboxplotvalue{lower quartile},0.5)!0.5\pgfkeysvalueof{/pgfplots/boxplot/box width}!(boxplot box cs:\pgfplotsboxplotvalue{lower quartile},0)$)
        rectangle
        ($(boxplot box cs:\pgfplotsboxplotvalue{upper quartile},0.5)!0.5\pgfkeysvalueof{/pgfplots/boxplot/box width}!(boxplot box cs:\pgfplotsboxplotvalue{upper quartile},1)$)
      ;
    }
  },
}

\author{
  Alexander Marx\\
  Max Planck Institute for Informatics, \\
  and Saarland University\\
  Saarbr{\"{u}}cken, Germany\\
  \texttt{amarx@mpi-inf.mpg.de} \\
   \And
   Jilles Vreeken \\
   CISPA Helmholtz Center for Information Security,\\
   and Max Planck Institute for Informatics\\
   Saarbr{\"{u}}cken, Germany \\
   \texttt{jv@cispa.saarland} \\
}

\begin{document}

\maketitle

\begin{abstract}
Testing for conditional independence is a core aspect of constraint-based causal discovery. Although commonly used tests are perfect in theory, they often fail to reject independence in practice, especially when conditioning on multiple variables.

We focus on discrete data and propose a new test based on the notion of algorithmic independence that we instantiate using stochastic complexity. Amongst others, we show that our proposed test, \SCI, is an asymptotically unbiased as well as $L_2$ consistent estimator for conditional mutual information (\CMI). 
Further, we show that \SCI can be reformulated to find a sensible threshold for \CMI that works well on limited samples. Empirical evaluation shows that \SCI has a lower type II error than commonly used tests. As a result, we obtain a higher recall when we use \SCI in causal discovery algorithms, \textit{without} compromising the precision.
\end{abstract}

\section{Introduction}
\label{sec:intro}

Testing for conditional independence plays a key role in causal discovery~\citep{spirtes:00:book}. If the true probability distribution of the observed data is faithful to the underlying causal graph, conditional independence tests can be used to recover the undirected causal network. In essence,
under the faithfulness assumption~\citep{spirtes:00:book} finding that two random variables $X$ and $Y$ are conditionally independent given a set of random variables $Z$, denoted as $X \independent Y \mid Z$, implies that there is no direct causal link between $X$ and $Y$.

As an example, consider Figure~\ref{fig:d_separation}. Nodes $F$ and $T$ are d-separated given $D,E$. Based on the faithfulness assumption, we can identify this from i.i.d. samples of the joint distribution, as $F$ will be independent of $T$ given $D,E$. In contrast, $D \not \independent T \mid E,F$, as well as $E \not \independent T \mid D,F$.

Conditional independence testing is also important for recovering the Markov blanket of a target node---i.e. the minimal set of variables, conditioned on which all other variables are independent of the target~\citep{pearl:88:firstmb}. There exist classic algorithms that find the correct Markov blanket with provable guarantees~\citep{margaritis:00:gs,pena:07:pcmb}. These guarantees, however, only hold under the faithfulness assumption and given a \textit{perfect} independence test.

In this paper, we are not trying to improve these algorithms, but rather propose a new independence test to enhance their performance. Recently a lot of work focuses on tests for continuous data; methods ranging from approximating continuous conditional mutual information~\citep{runge:18:continuouscmi} to kernel based methods~\citep{zhang:11:kernel}, we focus on discrete data.

\begin{figure}[t]
		\centering
		\includegraphics[]{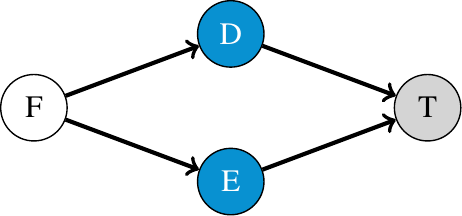}
	\caption{[d-Separation] Given the above causal DAG it holds that $F \independent T \mid D,E$, or $F$ is d-separated of $T$ given $D, E$ under the faithfulness assumption. Note that $D \not\independent T \mid E,F$ and $E \not\independent T \mid D,F$.}
	\label{fig:d_separation}
\end{figure}

For discrete data, two tests are frequently used in practice, the $G^2$ test~\citep{aliferis:10:hiton:overview,schluter:14:survey:mb} and conditional mutual information (\CMI)~\citep{zhang:10:iamb:lambda}. While the former is theoretically sound, it is very restrictive as it has a high sample complexity; especially when conditioning on multiple random variables. When used in algorithms to find the Markov blanket, for example, this leads to low recall, as there it is necessary to condition on larger sets of variables. 

If we had access to the true distributions, conditional mutual information would be the perfect criterium for conditional independence. Estimating \CMI purely from limited observational data leads, however, to discovering spurious dependencies---in fact, it is likely to find no independence at all~\citep{zhang:10:iamb:lambda}. To use \CMI in practice, it is therefore necessary to set a threshold. This is not an easy task, as the threshold should depend on both the domain sizes of the involved variables as well as the sample size~\citep{goebel:05:gamma}. Recently, \cite{canonne:18:sublinear:samples} showed that instead of exponentially many samples, theoretically \CMI has only a sub-linear sample complexity, although an algorithm is not provided.
Closest to our approach is the work of \cite{goebel:05:gamma} and \cite{suzuki:16:jic}. The former show that the empirical mutual information follows the gamma distribution, which allows them to define a threshold based on the domain sizes of the variables and the sample size. The latter employs an asymptotic formulation to determine the the threshold for \CMI.

The main problem of existing tests is that these struggle to find the right balance for limited data: either they are too restrictive and declare everything as independent or not restrictive enough and do not find any independence. To tackle this problem, we build upon algorithmic conditional independence, which has the advantage that we not only consider the statistical dependence, but also the complexity of the distribution. Although algorithmic independence is not computable, we can instantiate this ideal formulation with stochastic complexity. In essence, we compute stochastic complexity using either factorized or quotient normalized maximum likelihood (fNML and qNML)~\citep{silander:08:nml:bayesnet,silander:18:qnml}, and formulate \SCI, the \textit{\textbf{S}tochastic complexity based \textbf{C}onditional \textbf{I}ndependence criterium}.

Importantly, we show that we can reformulate \SCI to find a natural threshold for \CMI that works very well given limited data and diminishes given enough data. In the limit, we prove that \SCI is an asymptotically unbiased and $L_2$ consistent estimator of \CMI . For limited data, we find that the qNML threshold behaves similar to \cite{goebel:05:gamma}---i.e. it considers the sample size as well as the dimensionality of the data. The fNML threshold, however, additionally considers the estimated probability mass functions of the conditioning variables. In practice, this reduces the type II error. Moreover, when applying \SCI based on fNML in constraint based causal discovery algorithms, we observe a higher precision and recall than related tests. In addition, in our empirical evaluation \SCI shows a sub-linear sample complexity.

In this work we build upon and extend the basic ideas we first presented as~\citep{marx:18:climb}. Here we specifically focus on the theory and properties of using stochastic complexity for measuring conditional independence. Those readers that are interested in how SCI can be used in the discovery of directed Markov blankets we refer to~\citep{marx:18:climb}.

For conciseness, we postpone some proofs and experiments to the supplemental material. For reproducibility of our experiments we make our code available online\footnote{\oururl} and released an efficient version of \SCI in the R-package \textit{SCCI}.

\section{Conditional Independence Testing}
\label{sec:preliminaries} 

In this section, we introduce the notation and give brief introductions to both standard statistical conditional independence testing, as well as to the notion of algorithmic conditional independence. 

Given three possibly multivariate random variables $X$, $Y$ and $Z$, our goal is to test the conditional independence hypothesis $H_0 \colon X \independent Y \mid Z$ against the general alternative $H_1 \colon X \not \independent Y \mid Z$. The main goal of a good independence test is to minimize the type I and type II error.  The type I error is defined as falsely rejecting the null hypothesis and the type II error is defined as falsely accepting the null hypothesis.

A well known theoretical measure for conditional independence is conditional mutual information based on Shannon entropy~\citep{cover:06:elements}.

\begin{definition}
Given random variables $X$, $Y$ and $Z$. If
\begin{equation}
I(X ; Y \mid Z) := H(X \mid Z) - H(X \mid Z,Y) = 0
\label{eq:shannonindep}
\end{equation}
then $X$ and $Y$ are called statistically independent given $Z$. 
\end{definition}

In theory, conditional mutual information (\CMI) works perfectly as an independence test for discrete data. However, this only holds if we are given the true distributions of the random variables. In practice, those are not given. On a limited sample the plug-in estimator tends to underestimates conditional entropies, and as a consequence, the conditional mutual information is overestimated---even for completely independent data, as in the following Example.

\begin{example}
\label{ex:example}
Given three random variables $X_1$, $X_2$ and $Y$, with resp. domain sizes $1 \, 000, 8$ and $4$. Suppose that we are given $1 \, 000$ samples over their joint distribution and find that $\hat{H}(Y \mid X_1) = \hat{H}(Y \mid X_2) = 0$. That is, $Y$ is a deterministic function of $X_1$, as well as of $X_2$. However, as $|\mathcal{X}_1| = 1 \, 000$, and given only $1 \, 000$ samples, it is likely that we will have only a single sample for each $v \in \mathcal{X}_1$. That is, finding that $\hat{H}(Y \mid X_1) = 0$ is likely due to the limited amount of samples, rather than that it depicts a true (functional) dependency, while $\hat{H}(Y \mid X_2) = 0$ is more likely to be due to a true dependency, since the number of samples $n \gg |\mathcal{X}_2|$---i.e. we have more evidence.
\end{example}

A possible solution is to set a threshold $t$ such that $X \independent Y \mid Z$ if $I(X ; Y \mid Z) \le t$. Setting $t$ is, however, not an easy task, as $t$ is dependent on the quality of the entropy estimate, which by itself strongly depends on the complexity of the distribution and the given number of samples.
Instead, to avoid this problem altogether, we will base our test on the notion of \emph{algorithmic} independence. 

\subsection{Algorithmic Independence}

To define algorithmic independence, we need to give a brief introduction to Kolmogorov complexity. 
The Kolmogorov complexity of a finite binary string $x$ is the length of the shortest binary program $p^*$ for a universal Turing machine $\mathcal{U}$ that generates $x$, and then halts~\citep{kolmogorov:65:information,vitanyi:93:book}. Formally, we have
\[
K(x) = \min \{ |p| \mid p \in \{0,1\}^*, \mathcal{U}(p) = x \} \; .
\]
That is, program $p^*$ is the most succinct \emph{algorithmic} description of $x$, or in other words, the ultimate lossless compressor for that string. To define algorithmic independence, we will also need conditional Kolmogorov complexity, $K(x \mid y) \leq K(x)$, which is again the length of the shortest binary program $p^*$ that generates $x$, and halts, but now given $y$ as input for free. 

By definition, Kolmogorov complexity makes maximal use of any effective structure in $x$; structure that can be expressed more succinctly algorithmically than by printing it verbatim. As such it is the theoretical optimal measure for complexity.  In this point, algorithmic independence differs from statistical independence. In contrast to purely considering the dependency between random variables, it also considers the complexity of the process behind the dependency. 

Let us consider Example~\ref{ex:example} again and let $x_1$, $x_2$, and $y$ be the binary strings representing $X_1, X_2$ and $Y$. As $Y$ can be expressed as a deterministic function of $X_1$ or $X_2$, $K(y \mid x_1)$ and $K(y \mid x_2)$ reduce to the programs describing the corresponding function. As the domain size of $X_2$ is $8$ and $|\mathcal{Y}| = 4$, the program to describe $Y$ from $X_2$ only has to describe the mapping from $8$ to $4$ values, which will be shorter than describing a mapping from $X_1$ to $Y$, since $|\mathcal{X}_1|=1 \, 000$---i.e. $K(y \mid x_2) \le K(y \mid x_1)$ in contrast $\hat{H}(Y \mid X_1) = \hat{H}(Y \mid X_2)$. To reject $Y \independent X \mid Z$, we test whether providing the information of $X$ leads to a shorter program than only knowing $Z$. Formally, we define algorithmic conditional independence as follows~\citep{chaitin:75:algindepcond}.
\begin{definition}
Given the strings $x,y$ and $z$, We 
write $z^*$ to denote the shortest program for $z$, and analogously $(z,y)^*$ for the shortest program for the concatenation of $z$ and $y$. If
\begin{equation}
I_A(x ; y \mid z) := K(x \mid z^*) - K(x \mid (z,y)^*) \stackrel{+}{=} 0 
\label{eq:algorithmicindep}
\end{equation}
holds up to an additive constant that is independent of the data, then $x$ and $y$ are called algorithmically independent given $z$. 
\end{definition}

Due to the halting problem Kolmogorov complexity is not computable, however, nor approximable up to arbitrary precision~\citep{vitanyi:93:book}. The Minimum Description Length (MDL) principle~\citep{grunwald:07:book} provides a statistically well-founded approach to approximate it from above. For discrete data, this means we can use the stochastic complexity for multinomials~\citep{kontkanen:07:histo}, which belongs to the class of refined MDL codes.

\section{Stochastic Complexity for Multinomials}
\label{sec:stochasticcomplexity}

Given $n$ samples of a discrete univariate random variable $X$ with a domain $\mathcal{X}$ of $|\mathcal{X}| = k$ distinct values, $x^n \in
\mathcal{X}^n$, 
let $\hat{\theta}(x^n)$ denote the maximum likelihood estimator for $x^n$. \citet{shtarkov:87:universal} defined the mini-max optimal \textit{normalized maximum likelihood (NML)}
\begin{equation}
P_{\nml}(x^n \mid \models_k) = \frac{P(x^n \mid \hat{\theta}(x^n), \models_k)}{\regret_{\models_k}^n} \; , \label{eq:pnml}
\end{equation}
where the normalizing factor, or regret $\regret_{\models_k}^n$, relative to the model class $\models_k$ is defined as
\begin{equation}
\regret_{\models_k}^n = \sum_{\tilde{x}^n \in \mathcal{X}^n} P(\tilde{x}^n \mid \hat{\theta}(\tilde{x}^n), \models_k) \, . \label{eq:regret}
\end{equation}
The sum goes over every possible $\tilde{x}^n$ over the domain of $X$, and for each considers the maximum likelihood for that data given model class $\models_k$. Whenever clear from context, we drop the model class to simplify the notation---i.e. we write $P_{\nml}(x^n)$ for $P_{\nml}(x^n \mid \models_k)$ and $\regret_{k}^n$ to refer to $\regret_{\models_k}^n$.

For discrete data, assuming a multinomial distribution, we can rewrite Eq.~\eqref{eq:pnml} as~\citep{kontkanen:07:histo}
\[
P_{\nml}(x^n) = \frac{\prod_{j=1}^k \left(\frac{|v_j|}{n} \right)^{|v_j|}}{\regret_{k}^n} \; ,
\]
writing $|v_j|$ for the frequency of value $v_j$ in $x^n$, resp.\ Eq.~\eqref{eq:regret} as
\[
\regret_{k}^n = \sum_{|v_1| + \cdots + |v_k| = n} \frac{n!}{|v_1|! \cdots |v_k|!} \prod_{j=1}^k \left( \frac{|v_j|}{n} \right)^{|v_j|} \ .
\]
\cite{mononen:08:sub-lin-stoch-comp} derived a formula to calculate the regret in sub-linear time, meaning that the whole formula can be computed in linear time w.r.t. $n$.

We obtain the stochastic complexity for $x^n$ by simply taking the negative logarithm of $P_{\nml}$, 
which decomposes into a Shannon-entropy and the log regret
\begin{align}
\SC(x^n) &= - \log P_{\nml}(x^n) \; ,\\
&= n \hat{H}(x^n) + \log \regret_{k}^n \; . \label{eq:sc:unc}
\end{align}

\subsection{Conditional Stochastic Complexity}

Conditional stochastic complexity can be defined in different ways. We consider factorized normalized maximum likelihood (fNML)~\citep{silander:08:nml:bayesnet} and quotient normalized maximum likelihood (qNML)~\citep{silander:18:qnml}, which are equivalent except for the regret terms.

Given $x^n$ and $y^n$ drawn from the joint distribution of two random variables $X$ and $Y$, where $k$ is the size of the domain of $X$. Conditional stochastic complexity using the fNML formulation is defined as
\begin{align}
\SC_f(x^n \mid y^n) &= \sum_{v \in \mathcal{Y}} - \log P_{\nml}(x^n \mid y^n = v) \\
&= \sum_{v \in \mathcal{Y}} |v| \hat{H}(x^n \mid y^n \!= \!v) + \sum_{v\in \mathcal{Y}} \log \regret_{k}^{|v|}, \; \; \label{eq:sc:cond}
\end{align}
where $y^n=v$ denotes the set of samples for which $Y=v$, $\mathcal{Y}$ the domain of $Y$ with domain size $l$, and $|v|$ the frequency of a value $v$ in $y^n$. 

Analogously, we can define conditional stochastic complexity $S_q$ using qNML~\citep{silander:18:qnml}. We prove all important properties of our independence test for both fNML and qNML definitions, but for conciseness, and because $S_f$ performs superior in our experiments, we postpone the definition of $S_q$ and the related proofs to the supplemental material. 

In the following, we always consider the sample size $n$ and slightly abuse the notation by replacing $\SC(x^n)$ by $\SC(X)$, similar so for the conditional case. We refer to conditional stochastic complexity as $\SC$ and only use $\SC_f$ or $\SC_q$ whenever there is a conceptual difference. In addition, we refer to the regret terms of the conditional $\SC(X \mid Z)$ as $\Penalty(X \mid Z)$, where 
\[
\Penalty_f(X \mid Z) = \sum_{z \in \mathcal{Z}} \log \regret_{|\mathcal{X}|}^{|z|} \; .
\]

Next, we show that the regret term is log-concave in $n$, which is a property we need later on.

\begin{lemma}
\label{lemma:log:concave}
For $n \ge 1$, the regret term $\regret_{k}^n$ of the multinomial stochastic complexity of a random variable with a domain size of $k \ge 2$ is log-concave in $n$.
\end{lemma}

For conciseness, we postpone the proof of Lemma~\ref{lemma:log:concave} to the supplementary material. Based on Lemma~\ref{lemma:log:concave} we can now introduce Theorem~\ref{th:fmonotone} that is essential for our proposed independence test.

\begin{theorem}
\label{th:fmonotone}
Given three random variables $X$, $Y$ and $Z$, it holds that $\Penalty_f(X \mid Z) \le \Penalty_f(X \mid Z,Y)$.
\end{theorem}

\begin{proof}
Consider that $Z$ contains $p$ distinct value combinations $\{ r_1, \dots, r_p \}$. If we add $Y$ to $Z$, the number of distinct value combinations, $\{ l_1, \dots, \l_q \}$, increases to $q$, where $p \le q$. Consequently, to show that Theorem~\ref{th:fmonotone} is true, it suffices to show that
\begin{equation}
\sum_{i = 1}^p \log \regret_{k}^{|r_i|} \le \sum_{j = 1}^q \log \regret_{k}^{|l_j|} \, \label{eq:sub:additivity:preparation}
\end{equation}
where $\sum_{i=1}^p |r_i| = \sum_{j=1}^q |l_j| = n$. Next, consider w.l.o.g. that each value combination $\{r_i\}_{i=1, \dots, p}$ is mapped to one or more value combinations in $\{ l_1, \dots, \l_q \}$. Hence, Eq.~\eqref{eq:sub:additivity:preparation} holds, if the $\log \regret_{k}^n$ is sub-additive in $n$. Since we know from Lemma~\ref{lemma:log:concave} that the regret term is log-concave in $n$, sub-additivity follows by definition.
\end{proof}

Now that we have all the necessary tools, we can define our independence test in the next section.

\section{Stochastic Complexity based Conditional Independence}
\label{sec:independence}

With the above, we can now formulate our new conditional independence test, which we will refer to as the \textit{\textbf{S}tochastic complexity based \textbf{C}onditional \textbf{I}ndependence criterium}, or \SCI for short.
\begin{definition}
Let $X$, $Y$ and $Z$ be random variables. We say that $X \independent Y \mid Z$, if
\begin{equation}
\SCI(X;Y \mid Z) := \SC(X \mid Z) - \SC(X \mid Z,Y) \le 0 \; . \label{eq:isc}
\end{equation}
\end{definition}
In particular, Eq.~\refeq{eq:isc} can be rewritten as
\begin{align}
\SCI(X;Y \mid Z) &= n \cdot I(X;Y \mid Z) \\
&+ \Penalty(X \mid Z) - \Penalty(X \mid Z,Y) \; . \label{eq:iscdetail}
\end{align}
From this formulation, we see that the regret terms formulate a threshold $t_{\SC}$ for conditional mutual information, where $t_{\SC} = \Penalty(X \mid Z,Y) - \Penalty(X \mid Z)$. From Theorem~\ref{th:fmonotone} we know that if we instantiate \SCI using fNML that $\Penalty(X \mid Z,Y) - \Penalty(X \mid Z) \ge 0$. Hence, $Y$ has to provide a significant gain such that $X \not \independent Y \mid Z$---i.e. we need $\hat{H}(X \mid Z) - \hat{H}(X \mid Z, Y) > t_{\SC} / n$. 

Next, we show how we can use \SCI in practice by formulating it using fNML.

\subsection{Factorized SCI}

To formulate our independence test based on factorized normalized maximum likelihood, we have to revisit the regret terms again. In particular, $\Penalty_f(X \mid Z)$ is only equal to $\Penalty_f(Y \mid Z)$, when the domain size of $X$ is equal to the domain of $Y$. Further, $\Penalty_f(X \mid Z) - \Penalty_f(X \mid Z,Y)$ is not guaranteed to be equal to $\Penalty_f(Y \mid Z) - \Penalty_f(Y \mid Z,X)$. As a consequence, 
\[
I_{\SC}^f(X;Y \mid Z) := \SC_f(X \mid Z) - \SC_f(X \mid Z,Y)
\]
is not always equal to 
\[
I_{\SC}^f(Y;X \mid Z) := \SC_f(Y \mid Z) - \SC_f(Y \mid Z,X) \; .
\]
To achieve symmetry, we formulate $\SCI_f$ as
\begin{equation}
\SCI_f(X;Y \mid Z) := \max \{ I_{\SC}^f(X;Y \mid Z), I_{\SC}^f(Y;X \mid Z) \}
\end{equation} 
and say that $X \independent Y \mid Z$, if $\SCI_f(X;Y \mid Z) \le 0$.

There are other ways to achieve such symmetry, such as via an alternative definition of conditional mutual information. However, as we show in detail in the supplementary, there exist serious issues with these alternatives when instantiated with fNML. 

Instead of the exact fNML formulation, it is also possible to use the asymptotic approximation of stochastic complexity~\citep{rissanen:96:fisher}, which was done by \cite{suzuki:16:jic} to approximate \CMI. In practice, the corresponding test (\JIC) is, however, very restrictive, which leads to low recall.

In the next section, we show the main properties for \SCI using fNML. 
Thereafter, we compare $\SCI$ to \CMI using the threshold based on the gamma distribution~\citep{goebel:05:gamma}, and empirically evaluate the sample complexity of \SCI.

\subsection{Properties of SCI}

In the following, for readability, we write \SCI to refer to properties that hold for both versions of \SCI, $\SCI_f$ and $\SCI_q$.

First, we show that if $X \independent Y \mid Z$, we have that $\SCI(X;Y \mid Z) \le 0$. Then, we prove that $\frac{1}{n}\SCI$ is an asymptotically unbiased estimator of conditional mutual information and is $L_2$ consistent. Note that by dividing $\SCI$ by $n$ we do not change the decisions we make as long as $n < \infty$. Since we only accept $H_0$ if $\SCI \le 0$, any positive output will still be $> 0$ after dividing it by $n$. 

\begin{theorem}
\label{th:sci_indep}
If $X \independent Y \mid Z$, $\SCI(X;Y\mid Z) \le 0$.
\end{theorem}

\begin{proof}
W.l.o.g. we can assume that $I_{\SC}^f(X;Y \mid Z) >= I_{\SC}^f(Y;X \mid Z)$. Based on this, it suffices to show that $I_{\SC}^f(X;Y \mid Z) \le 0$
if $X \independent Y \mid Z$. As the first part of $I_{\SC}^f$ consists of $n \cdot I(X;Y \mid Z)$, it will be zero by definition. We know that $\Penalty_f(X \mid Z) - \Penalty_f(X \mid Z,Y) \le 0$ (Theorem~\ref{th:fmonotone}), which concludes the proof.
\end{proof}

Next, we show that $\frac{1}{n}\SCI$ converges against conditional mutual information and hence is an asymptotically unbiased estimator of conditional mutual information and is $L_2$ consistent to it.

\begin{lemma}
\label{lemma:scii}
Given three random variables $X$, $Y$ and $Z$, it holds that 
$\lim_{n \rightarrow \infty} \frac{1}{n} \SCI(X; Y \mid Z) = I(X; Y \mid Z)$.
\end{lemma}

\begin{proof}
To show the claim, we need to show that
\[
\lim_{n \rightarrow \infty} I(X;Y \mid Z) + \frac{1}{n} ( \Penalty(X \mid Z) - \Penalty(X \mid Z, Y) ) = 0 \; .
\]
The proof for $I_{\SC}^f(Y;X \mid Z)$ follows analogously. In essence, we need to show that $\frac{1}{n} ( \Penalty(X \mid Z) - \Penalty(X \mid Z, Y) )$ goes to zero as $n$ goes to infinity. From \cite{rissanen:96:fisher} we know that $\log \regret_{k}^n$ asymptotically behaves like $\frac{k-1}{2} \log n + \mathcal{O}(1)$. Hence, $\frac{1}{n} \Penalty(X \mid Z)$ and $\frac{1}{n} \Penalty(X \mid Z, Y)$ will approach zero if $n \rightarrow \infty$.
\end{proof} 

As a corollary to Lemma~\ref{lemma:scii} we find that $\frac{1}{n}\SCI$ is an asymptotically unbiased estimator of conditional mutual information and is $L_2$ consistent to it.

\begin{theorem}
\label{th:unbiased}
Let $X$, $Y$ and $Z$ be discrete random variables. Then $\lim_{n \to \infty} \mathbb{E}[\frac{1}{n} \SCI(X;Y | Z)] = I(X;Y | Z)$, i.e. $\frac{1}{n} \SCI$ is an asymptotically unbiased estimator for conditional mutual information.
\end{theorem}

\begin{theorem}
\label{th:l2}
Let $X$, $Y$ and $Z$ be discrete random variables. Then $\lim_{n \to \infty} \mathbb{E}[(\frac{1}{n} \SCI(X;Y | Z) - I(X;Y | Z))^2] = 0$ i.e. $\frac{1}{n} \SCI$ is an $L_2$ consistent estimator for conditional mutual information.
\end{theorem}

Next, we compare both of our tests to the findings of \cite{goebel:05:gamma}.

\subsection{Link to Gamma Distribution}
\label{sec:gamma}

\cite{goebel:05:gamma} estimate conditional mutual information through a second-order Taylor series and show that their estimator can be approximated with the gamma distribution. In particular, they state that  
\[
\hat{I}(X;Y \mid Z) \sim \Gamma \left( \frac{|\mathcal{Z}|}{2}(|\mathcal{X}|-1)(|\mathcal{Y}|-1), \frac{1}{n \ln 2} \right) \; ,
\]
where $\mathcal{X}$, $\mathcal{Y}$ and $\mathcal{Z}$ refer to the domains of $X$, $Y$ and $Z$. This means by selecting a significance threshold $\alpha$, we can derive a threshold for \CMI based on the gamma distribution---for convenience we call this threshold $t_{\Gamma}$. In the following, we compare $t_{\Gamma}$ against $t_{\SC} = \Penalty(X \mid Z,Y) - \Penalty(X \mid Z)$.

First of all, for qNML, like $t_{\Gamma}$, $t_{\SC}$  depends purely on the sample size and the domain sizes. However, we consider the difference in complexity between only conditioning $X$ on $Z$ and the complexity of conditioning $X$ on $Z$ and $Y$. For fNML, we have the additional aspect that the regret terms for both $\Penalty(X \mid Z)$ and $\Penalty(X \mid Z,Y)$ also relate to the probability mass functions of $Z$, and respectively the Cartesian product of $Z$ and $Y$. Recall that for $k$ being the size of the domain of $X$, we have that 
\[
\Penalty_f(X \mid Z) = \sum_{z \in Z} \log \regret_{k}^{|z|} \; .
\]
As $\regret_k^n$ is log-concave in $n$ (Lemma~\ref{lemma:log:concave}), $\Penalty_f(X \mid Z)$ is maximal if $Z$ is uniformly distributed---i.e. it is maximal when $H(Z)$ is maximal. This is a favourable property, as the probability that $Z$ is equal to $X$ is minimal for uniform $Z$, as stated in the following Lemma \citep{cover:06:elements}.

\begin{lemma}
If $X$ and $Y$ are i.i.d. with entropy $H(Y)$, then $P(Y = X) \ge 2^{-H(Y)}$ with equality if and only if $Y$ has a uniform distribution.
\end{lemma}

\begin{figure}[t]%
	\begin{minipage}[t]{.5\linewidth}
	\centering
	\includegraphics[]{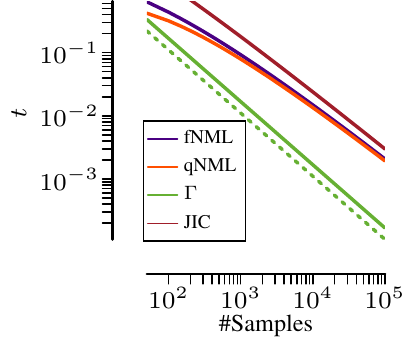}
	\end{minipage}%
	\begin{minipage}[t]{.5\linewidth}
	\centering
	\includegraphics[]{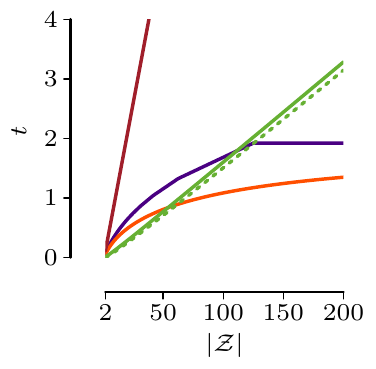}
	\end{minipage}%
	\caption{Threshold for $\CMI$ using fNML, qNML, \JIC and the gamma distribution with $\alpha=0.05$ (solid) and $\alpha=0.001$ (dashed) for different sample sizes and fixed domain sizes equal to four (left) and fixed sample size of $500$ and changing domain sizes (right).}
	\label{fig:thresholds}
\end{figure}

To elaborate the link between $t_{\Gamma}$ and $t_{\SC}$, we compare them empirically. In addition, we compare the results to the threshold provided from the \JIC test. First, we compare $t_{\Gamma}$ with $\alpha = 0.05$ and $\alpha = 0.001$ to $t_{\SC} / n$ for fNML and qNML, and \JIC on fixed domain sizes, with $|\mathcal{X}| = |\mathcal{Y}| = |\mathcal{Z}| = 4$ and varying the sample sizes (see Figure~\ref{fig:thresholds}). For fNML we computed the worst case threshold under the assumption that $Z$ is uniformly distributed. In general, the behaviour for each threshold is similar, whereas qNML, fNML and \JIC are more restrictive than $t_{\Gamma}$.

Next, we keep the sample size fix at $500$ and increase the domain sizes of $Z$ from $2$ to $200$, to simulate multiple variables in the conditioning set. Except to \JIC, which seems to overpenalize in this case, we observe that fNML is most restrictive until we reach a plateau when $|\mathcal{Z}| = 125$. This is due to the fact that $|\mathcal{Z}| |\mathcal{Y}| = 500$ and hence each data point is assigned to one value in the Cartesian product. We have that $\Penalty_f(X \mid Z, Y) = |\mathcal{Z}| |\mathcal{Y}| \regret_k^1$.

It is important to note, however, that the thresholds that we computed for fNML assume that $Z$ and $Y$ are uniformly distributed and $Y \independent Z$. In practice, when this requirement is not fulfilled, the regret term of fNML can be smaller than this value, since it is data dependent. In addition, it is possible that the number of distinct values that we observe from the joint distribution of $Z$ and $Y$ is smaller than their Cartesian product, which also reduces the difference in the regret terms for fNML.

\subsection{Empirical Sample Complexity}

In this section, we empirically evaluate the sample complexity of $\SCI_f$, where we focus on the type I error, i.e. $H_0 \colon X \independent Y \mid Z$ is true and hence $I(X;Y \mid Z) = 0$. We generate data accordingly and draw samples from the joint distribution, where we set $P(x,y,z) = \frac{1}{|\mathcal{X}||\mathcal{Y}||\mathcal{Z}|}$ for each value configuration $(x,y,z) \in \mathcal{X} \times \mathcal{Y} \times \mathcal{Z}$. Per sample size we draw $1 \, 000$ data sets and report the average absolute error for $\SCI_f$ and the empirical estimator of CMI, $\hat{I}$. We show the results for two cases in Fig.~\ref{fig:sample-complexity-lines}. We observe that in contrast to the empirical plug-in estimator $\hat{I}$, 
$\SCI_f$ quickly approaches zero, and that the difference is especially large for larger domain sizes.

\begin{figure}[t]%
	\begin{minipage}[t]{.5\linewidth}
	\centering
	\includegraphics[]{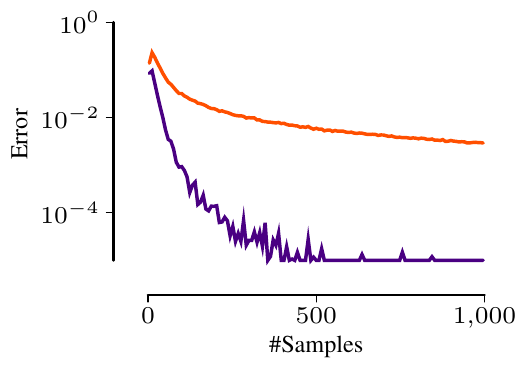}
	\end{minipage}%
	\begin{minipage}[t]{.5\linewidth}
	\centering
	\includegraphics[]{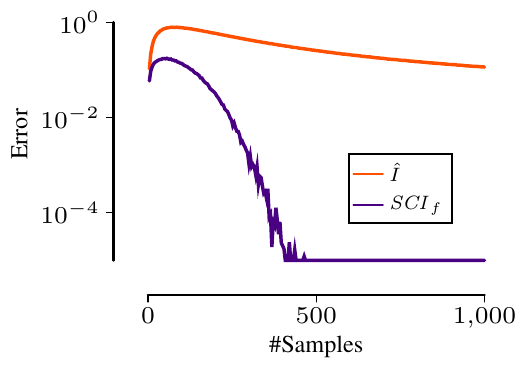}
	\end{minipage}%
	\caption{Error for $\SCI_f$ and $\hat{I}$ compared to $I$, where $I(X;Y | Z)\myeq0$. Left: $|\mathcal{X}|\myeq|\mathcal{Y}|\myeq2$ and $|\mathcal{Z}|\myeq4$. Right: $|\mathcal{X}|\myeq|\mathcal{Y}|\myeq4$ and $|\mathcal{Z}|\myeq16$. Values smaller than $10^{-5}$ are truncated to $10^{-5}$.} 
	\label{fig:sample-complexity-lines}
\end{figure}

In the supplemental material we give a more in depth analysis alltogether. Our evaluation suggest that the sample complexity is sub-linear. In particular, we find that the number of samples $n$ required such that $P(|\SCI_f^n(X; Y \mid Z) / n - I(X;Y \mid Z)| \ge \epsilon) \le \delta$, with $\epsilon = \delta = 0.05$ is smaller than $35 + 2 |\mathcal{X}| |\mathcal{Y}|^{2/3}(|\mathcal{Z}|+1)$.

To illustrate this, consider the left example in Figure~\ref{fig:sample-complexity-lines} again. We observe that for $\epsilon = \delta = 0.05$, $n$ needs to be at least $52$, which is smaller than the value from our empirical bound function, that is equal to $67$. If we require $\epsilon = 0.01$ and $\delta = 0.05$, we observe that $n$ must be at least $72$. In comparison, for $\hat{I}$, $n$ needs to be at least $140$ for $\epsilon = 0.05$ and $684$ for $\epsilon = 0.01$.

\subsection{Discussion}

The main idea for our independence test is to approximate conditional mutual information through algorithmic conditional independence. In particular, we estimate conditional entropy with stochastic complexity. We recommend $\SCI_f$, since the regret for the entropy term does not only depend on the sample size and the domain sizes of the corresponding random variables, but also on the probability mass function of the  conditioning variables. In particular, when fixing the domain sizes and the sample size, higher thresholds are assigned to conditioning variables that are unlikely to be equal to the target variable. 

By assuming a uniform distribution for the conditioning variables and hence eliminating this data dependence from $\SCI_f$, it behaves similar to $\SCI_q$ and \CMI where the threshold is derived from the gamma distribution~\citep{goebel:05:gamma}. $\SCI_f$ is more restrictive and the penalty terms of all three decrease exponentially w.r.t. the sample size.

$\SCI$ can also be extended for sparsification, as is possible to derive an analytical p-value for the significance of a decision using the no-hypercompression inequality~\citep{grunwald:07:book,marx:17:slope}. 

Last, note that as we here instantiate \SCI using stochastic complexity for multinomials, we implicitly assume that the data follows a multinomial distribution. In this light, it is important to note that stochastic complexity is a mini-max optimal refined MDL code~\citep{grunwald:07:book}. This means that for any data, we obtain a score that is within a constant term from the best score attainable given our model class. The experiments verify that indeed, \SCI performs very well, even when the data is sampled adversarially.

\section{Experiments}
\label{sec:experiments}

In this section, we empirically evaluate \SCI based on fNML and compare it to the alternative formulation using qNML. In addition, we compare it to the $G^2$ test from the \textit{pcalg} R package~\citep{kalisch:12:pcalg}, $\CMI_{\Gamma}$~\citep{goebel:05:gamma} and \JIC~\citep{suzuki:16:jic}.

\subsection{Identifying d-Separation}
\label{sec:dsep:test}

To test whether \SCI can reliably distinguish between independence and dependence, we generate data as depicted in Figure~\ref{fig:d_separation}, where we draw $F$ from a uniform distribution and model a dependency from $X$ to $Y$ by simply assigning uniformly at random each $x \in \mathcal{X}$ to a $y \in \mathcal{Y}$. We set the domain size for each variable to $4$ and generate data under various samples sizes ($100$--$2 \, 500$) and additive uniform noise settings ($0 \%$--$95 \%$). For each setup we generate $200$ data sets and assess the accuracy. In particular, we report the correct identifications of $F \independent T \mid D,E$ as the true positive rate and the false identifications $D \independent T \mid E,F$ or $E \independent T \mid D,F $ as false positive rate.\!\footnote{For $0\%$ noise, $F$ has all information about $D$ and $E$. Hence, in this specific case, $D \not \independent T \mid E,F$ and $E \not \independent T \mid D,F$ does not hold.} For the $G^2$ test and $\CMI_{\Gamma}$ we select $\alpha = 0.05$, however, we found no significant differences for $\alpha = 0.01$.

In the interest of space we only plot the accuracy of the best performing competitors in Figure~\ref{fig:indep_comp} and report the remaining results as well as the true and false positive rates for each approach in the supplemental material. Overall, we observe that $\SCI_f$ performs near perfect for less than $70 \%$ additive noise. When adding $70 \%$ or more noise, the type II error increases. Those results are even better than expected as from our empirical bound function we would suggest that at least $378$ samples are required to have reliable results for this data set. $\SCI_q$ has a similar but slightly worse performance. In contrast, $\CMI_{\Gamma}$ only performs well for less than $30 \%$ noise and fails to identify true independencies after more than $30 \%$ noise has been added, which leads to a high type I error. The $G^2$ test has problems with sample sizes up to $500$ and performs inconsistently given more than $35\%$ noise. Note that we forced $G^2$ to decide for every sample size, while the minimum number of samples recommended for $G^2$ on this data set would be $1 \, 440$, which corresponds to $10(|\mathcal{X}|-1)(|\mathcal{Y}|-1)(|\mathcal{Z}|)$.

\begin{figure}[t]%
	\begin{minipage}[t]{.5\linewidth}
	\centering
	\includegraphics[]{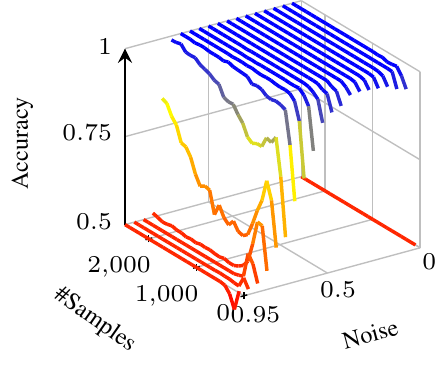}
	\subcaption{$\SCI_f$}
	\label{fig:dsep:scif}
	\end{minipage}%
	\begin{minipage}[t]{.5\linewidth}
	\centering
	\includegraphics[]{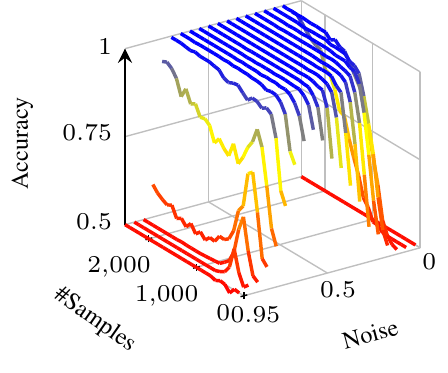}
	\subcaption{$\SCI_q$}
	\label{fig:dsep:sciq}
	\end{minipage}%
	\newline
	\begin{minipage}[t]{.5\linewidth}
	\centering
	\includegraphics[]{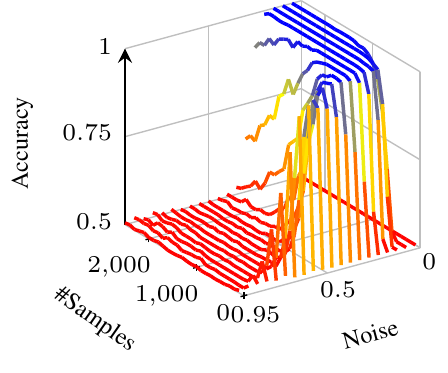}
	\subcaption{$\CMI_{\Gamma}$}
	\label{fig:dsep:gamma}
	\end{minipage}%
	\begin{minipage}[t]{.5\linewidth}
	\centering
	\includegraphics[]{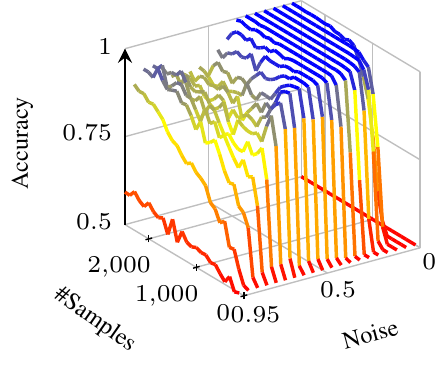}
	\subcaption{$G^2$}
	\label{fig:dsep:g2}
	\end{minipage}%
	\caption{[Higher is better] Accuracy of $\SCI_f$, $\SCI_q$, $\CMI_{\Gamma}$ and $G^2$ for identifying $d$-separation using varying samples sizes and additive noise percentages, where a noise level of $0.95$ refers to $95 \%$ additive noise.} 
	\label{fig:indep_comp}
\end{figure}

\subsection{Changing the Domain Size}
\label{ch:domain-sizes}

Using the same data generator as above, we now consider a different setup. We fix the sample size to $2 \; 000$ and use only $10 \%$ additive noise---a setup where all tests performed well. What we change is the domain size of the source $F$ from $2$ to $20$ while also restricting the domain sizes of the remaining variable to the same size. For each setup we generate $200$ data sets.

From the results in Figure~\ref{fig:domain_sizes_dsep} we can clearly see that only $\SCI_f$ is able to deal with larger domain sizes as for all other test, the false positive rate is at $100\%$ for larger domain sizes, resulting in an accuracy of $50 \%$.

\begin{figure}[t]%
	\centering
	\includegraphics[]{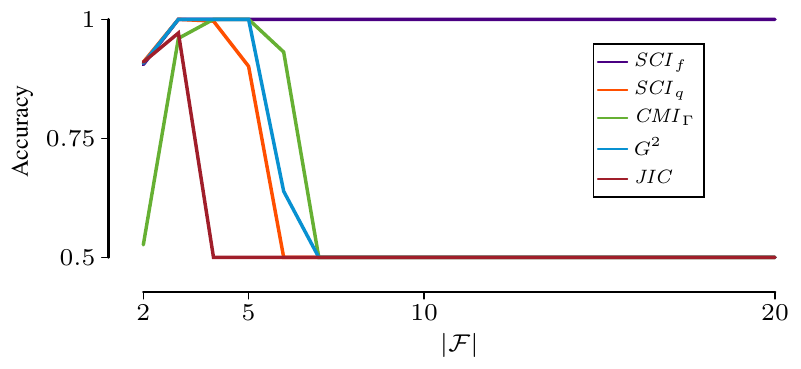}
	\caption{[Higher is better] d-Separation with $2 \, 000$ samples and $10 \%$ noise on different domain sizes of the source node $F$.}
	\label{fig:domain_sizes_dsep}
\end{figure}

\subsection{Plug and Play with SCI}

Last, we want to show how \SCI performs in practice. To do this, we run the stable PC algorithm~\citep{kalisch:12:pcalg,colombo:14:stablepc} on the \textit{Alarm} network~\citep{scutari:14:bnlearn} from which we generate data with different sample sizes and average over the results of $10$ runs for each sample size. We equip the stable PC algorithm with $\SCI_f$, $\SCI_q$, \JIC, $\CMI_{\Gamma}$ and the default, the $G^2$ test, and plot the average $F1$ score over the undirected graphs in Figure~\ref{fig:pc_small}. We observe that our proposed test, $\SCI_f$ outperforms the other tests for each sample size with a large margin and especially for small sample sizes.

\begin{figure}[h]
	\centering
	\includegraphics[]{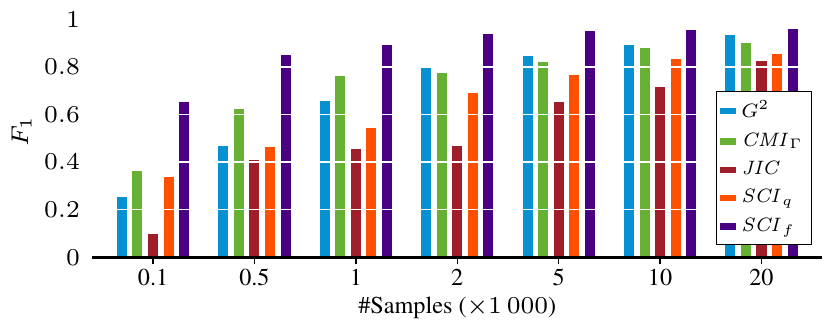}
	\caption{[Higher is better] $F_1$ score on undirected edges for stable PC using $\SCI_f$, $\SCI_q$, \JIC, $\CMI_{\Gamma}$ and $G^2$ on the \textit{Alarm} network for different sample sizes}
	\label{fig:pc_small}
\end{figure}

As a second practical test, we compute the Markov blanket for each node in the \textit{Alarm} network and report the precision and recall. To find the Markov blankets, we run the PCMB algorithm~\citep{pena:07:pcmb} with the four independence tests. We plot the precision and recall for each variant in Figure~\ref{fig:mb_comparison}. We observe that again $\SCI_f$ performs best---especially with regard to recall. As for Markov blankets of size $k$ it is necessary to condition on at least $k-1$ variables, this advantage in recall can be linked back to $\SCI_f$ being able to correctly detect dependencies for larger domain sizes.

\begin{figure}[t]%
	\begin{minipage}[t]{.5\linewidth}
	\centering
	\includegraphics[]{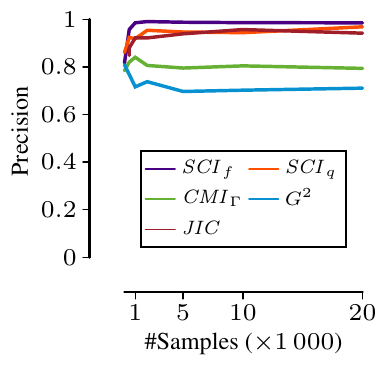}
	\end{minipage}%
	\begin{minipage}[t]{.5\linewidth}
	\centering
	\includegraphics[]{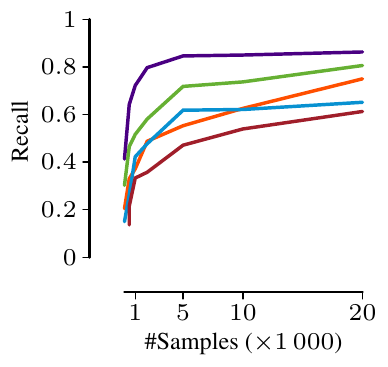}
	\end{minipage}%
	\caption{[Higher is better] Precision (left) and recall (right) for $\pcmb$ using $\SCI_f$, $\SCI_q$, \JIC, $\CMI_{\Gamma}$ and $G^2$ to identify all Markov blankets in the \textit{Alarm} network for different sample sizes.}
	\label{fig:mb_comparison}
\end{figure}

\section{Conclusion}
\label{sec:conclusion}

In this paper we introduced \SCI, a new conditional independence test for discrete data. We derive \SCI from algorithmic conditional independence and show that it is an unbiased asymptotic estimator for conditional mutual information (\CMI). Further, we show how to use \SCI to find a threshold for \CMI and compare it to thresholds drawn from the gamma distribution. 

In particular, we propose to instantiate \SCI using fNML as in contrast to using qNML or thresholds drawn from the gamma distribution, fNML does not only make use of the sample size and domain sizes of the involved variables, but also utilizes the empirical probability mass function of the conditioning variable. Moreover, we observe that $\SCI_f$ clearly outperforms its competitors on both synthetic and real world data. Last but not least, our empirical evaluations suggest that \SCI has a sub-linear sample complexity, which we would like to theoretically validate in future work.

\subsubsection*{Acknowledgements}
The authors would like to thank David Kaltenpoth for insightful discussions. 
Alexander Marx is supported by the International Max Planck Research School for Computer Science (IMPRS-CS). Both authors are
supported by the Cluster of Excellence on ``Multimodal Computing and Interaction'' within the Excellence Initiative of the German Federal Government.

\bibliography{bib/abbreviations,bib/bib-jilles,bib/bib-paper,bib/bib-alex}

\appendix
\label{sec:appendix}
\section{Extended Theory}
\label{app:concave}

\subsection{Proof of Lemma~\ref{lemma:log:concave}}
\begin{proof}
To improve the readability of this proof, we write $\regret_L^n$ as shorthand for $\regret_{\models_L}^n$ of a random variable with a domain size of $L$.

Since $n$ is an integer, each $\regret_L^n > 0$ and $\regret_L^0 = 1$, we can prove Lemma~\ref{lemma:log:concave}, by showing that the fraction $\regret_L^n / \regret_L^{n-1}$ is decreasing for $n \ge 1$, when $n$ increases.

We know from \cite{mononen:08:sub-lin-stoch-comp} that $\regret_L^n$ can be written as the sum 
\begin{equation}
\label{eq:regretlin}
\regret_L^n = \sum_{k=0}^n m(k,n) = \sum_{k=0}^n \frac{n\fallingfactorial{k}(L-1)\risingfactorial{k}}{n^kk!}  \; ,
\end{equation}
where $x\fallingfactorial{k}$ represent falling factorials and $x\risingfactorial{k}$ rising factorials. Further, they show that for fixed $n$ we can write $m(k,n)$ as
\begin{equation}
\label{eq:regret:rekursion:k}
m(k,n) = m(k-1,n) \frac{(n-k+1)(k+L-2)}{nk} \; ,
\end{equation}
where $m(0,n)$ is equal to $1$. It is easy to see that from $n=1$ to $n=2$ the fraction $\regret_L^n / \regret_L^{n-1}$ decreases, as $\regret_L^0 = 1$, $\regret_L^1 = L$ and $\regret_L^2 = L + L(L-1)/2$. In the following, we will show the general case. We rewrite the fraction as follows.
\begin{align}
\frac{\regret_L^n}{\regret_L^{n-1}} &= \frac{\sum_{k=0}^n m(k,n)}{\sum_{k=0}^{n-1} m(k,n-1)} \\
&= \frac{\sum_{k=0}^{n-1} m(k,n)}{\sum_{k=0}^{n-1} m(k,n-1)} + \frac{m(n,n)}{ \sum_{k=0}^{n-1} m(k,n-1)} \label{eq:begin:proof}
\end{align}
Next, we will show that both parts of the sum in Eq.~\refeq{eq:begin:proof} are decreasing when $n$ increases. We start with the left part, which we rewrite to
\begin{align}
\frac{\sum_{k=0}^{n-1} m(k,n)}{\sum_{k=0}^{n-1} m(k,n-1)} &= \frac{\sum_{k=0}^{n-1} m(k,n-1) + \sum_{k=0}^{n-1} \left( m(k,n)- m(k,n-1) \right) }{\sum_{k=0}^{n-1} m(k,n-1)} \\
&= 1+ \frac{\sum_{k=0}^{n-1} \frac{(L-1)\risingfactorial{k}}{k!} \left( \frac{n\fallingfactorial{k}}{n^k} - \frac{(n-1)\fallingfactorial{k}}{(n-1)^k} \right)}{\sum_{k=0}^{n-1} m(k,n-1)} \; . \label{eq:step2}
\end{align}
When $n$ increases, each term of the sum in the numerator in Eq.~\refeq{eq:step2} decreases, while each element of the sum in the denominator increases. Hence, the whole term is decreasing. In the next step, we show that the right term in Eq.~\refeq{eq:begin:proof} also decreases when $n$ increases. It holds that
\[
\frac{m(n,n)}{ \sum_{k=0}^{n-1} m(k,n-1)} \ge \frac{m(n,n)}{m(n-1,n-1)} \; .
\]
Using Eq.~\refeq{eq:regret:rekursion:k} we can reformulate the term as follows.
\begin{equation}
\frac{\frac{n+L-2}{n^2} m(n-1,n)}{m(n-1,n-1)} = \frac{n+L-2}{n^2} \left( 1 + \frac{m(n-1,n) - m(n-1,n-1)}{m(n-1,n-1)} \right)
\end{equation}
After rewriting, we have that $\frac{n+L-2}{n^2}$ is definitely decreasing with increasing $n$. For the right part of the product, we can argue the same way as for Eq.~\refeq{eq:step2}. Hence the whole term is decreasing, which concludes the proof.
\end{proof}

\subsection{Quotient SCI}

Conditional stochastic complexity can also be defined via quotient normalized maximum likelihood (qNML), which is defined as follows
\begin{equation}
\SC_q(x^n \mid y^n) = \sum_{v \in \mathcal{Y}} |v| \hat{H}(x^n \mid y^n \!= \!v) + \log \frac{\regret_{|\mathcal{X}| \cdot |\mathcal{Y}|}^n}{\regret_{|\mathcal{Y}|}^n} \; .
\end{equation}
We refer to the regret term of $\SC_q(X \mid Z)$ with
\[
\Penalty_q(X \mid Z) = \log \frac{\regret_{|\mathcal{X}| \cdot |\mathcal{Z}|}^n}{\regret_{|\mathcal{Z}|}^n} \; .
\]
Analogously to Theorem~\ref{th:fmonotone} for fNML, we can define the following theorem for qNML.
\begin{theorem}
\label{th:qmonotone}
Given three random variables $X$, $Y$ and $Z$, it holds that $\Penalty_q(X \mid Z) \le \Penalty_q(X \mid Z,Y)$.
\end{theorem}
\begin{proof}
Consider $n$ samples of three random variables $X$, $Y$ and $Z$, with corresponding domain sizes $k$, $p$ and $q$. It should hold that
\begin{align}
\Penalty_q(X \mid Z) &\le \Penalty_q(X \mid Z,Y) \\
\Leftrightarrow \log \frac{\regret_{kq}^n}{\regret_{q}^n} & \le \log \frac{\regret_{kpq}^n}{\regret_{pq}^n} \; .
\end{align}
We know from \cite{silander:18:qnml} that for $p \in \mathbb{N}, p \ge 2$, the function $q \mapsto \frac{\regret_{p \cdot q}^n}{\regret_q^n}$ is increasing for every $q \ge 2$. This suffices to proof the statement above.
\end{proof}

To formulate \SCI using quotient normalized maximum likelihood, we can straightforwardly replace $\SC$ with $\SC_q$ in the independence criterium---i.e.
\begin{equation}
\SCI_q(X;Y \mid Z) := \SC_q(X \mid Z) - \SC_q(X \mid Z,Y) 
\end{equation}
and say that $X \independent Y \mid Z$, if $\SCI_q(X;Y \mid Z) \le 0$. By writing down the regret terms for $\SCI_q(X;Y \mid Z)$ and $\SCI_q(Y;X \mid Z)$, we can see that they are equal and hence $\SCI_q$ is symmetric, that is, $\SCI_q(X;Y \mid Z) = \SCI_q(Y;X \mid Z)$.

Since we showed that Theorem~\ref{th:qmonotone} holds for qNML, Theorems~\ref{th:sci_indep}-\ref{th:l2} can also be proven for qNML using the same arguments as for fNML.

\subsection{Alternative Symmetry Correction for Factorized SCI}

To instantiate \SCI using fNML, we take the maximum between $I_{\SC}^f(X;Y \mid Z)$ and $I_{\SC}^f(Y;X \mid Z)$ to achieve symmetry. We could also achieve symmetry when we base our formulation on an alternative formulation of conditional mutual information, that is
\begin{equation}
\CMI(X;Y \mid Z) = H(X \mid Z) + H(Y \mid Z) - H(X,Y \mid Z) \; . \label{eq:cmialternative}
\end{equation}
In particular, we formulate our alternative test by replacing the conditional entropies in Eq.~\refeq{eq:cmialternative} with stochastic complexity based on fNML
\begin{equation}
\SCI_{\textit{fs}}(X;Y \mid Z) = \SC_f(X \mid Z) + \SC_f(Y \mid Z) - \SC_f(X,Y \mid Z) \; .
\end{equation}
By writing down the regret terms, we see that $\SCI_{\textit{fs}}(X;Y \mid Z) = \SCI_{\textit{fs}}(Y;X \mid Z)$. In particular, if we only consider the regret terms, we get
\begin{equation}
\sum_{z \in Z} \left( \regret_{|\mathcal{X}|}^{|z|} + \regret_{|\mathcal{Y}|}^{|z|} - \regret_{|\mathcal{X}||\mathcal{Y}|}^{|z|}  \right) \; . \label{eq:scifs}
\end{equation}
From Eq.~\refeq{eq:scifs} we see that all regret terms depend on the factorization given $Z$. For $I_{\SC}^f(X;Y \mid Z)$, however, we compare the factorizations of $X$ given only $Z$ to the one given $Z$ and $Y$, and similarly so for $I_{\SC}^f(Y;X \mid Z)$. In addition, for $\SCI_f$ all regret terms correspond to the same domain, either to the domain of $X$ or $Y$, whereas for $\SCI_{\textit{fs}}$ the regret terms are based on $X$, $Y$ and the Cartesian product of them. Since the last regret term of $\SCI_{\textit{fs}}$ is based on the Cartesian product of $X$ and $Y$ it performs worse than $\SCI_f$ for large domain sizes. This can also be seen in Figure~\ref{fig:domain-sizes-dsep-fs}, for which we conducted the same experiment as in Section~\ref{ch:domain-sizes}, but also applied $\SCI_{\textit{fs}}$. $\SCI_q$ exhibits similar behaviour like $\SCI_{\textit{fs}}$, as it also considers products of domain sizes.

There also exists a third way to formulate \CMI, i.e.
\begin{equation}
\CMI(X;Y \mid Z) = H(X,Z) + H(Y,Z) - H(X,Y,Z) - H(Z) \; . \label{eq:alternative2}
\end{equation}
When we replace all entropy terms with the stochastic complexity in Eq.~\refeq{eq:alternative2}, we would get an equivalent formulation to $\SCI_q$, as the regret terms would sum up to exactly the same values. Hence, we do not elaborate further on this alternative.

\begin{figure}[t]%
\centering
	\includegraphics[]{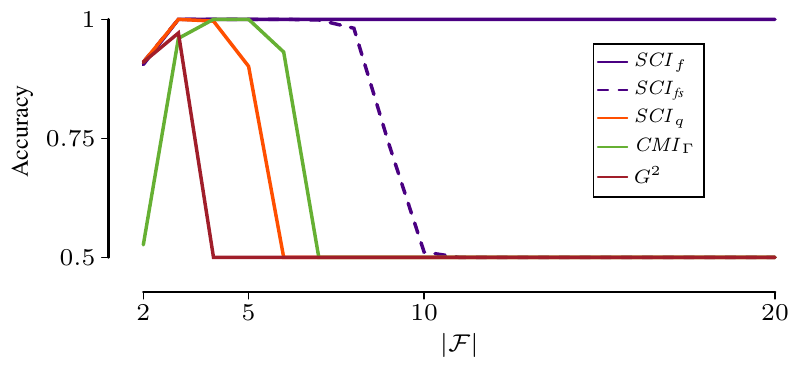}
	\caption{[Higher is better] d-Separation with $2 \, 000$ samples and $10 \%$ noise on different domain sizes of the source node $F$.}
	\label{fig:domain-sizes-dsep-fs}
\end{figure}

\section{Experiments}
\label{app:experiments}

In this section, we provide more details to the true positive and false positive rates w.r.t. d-separation. Further, we show how well \SCI and its competitors can recover multiple parents with and without additional noise variables in the conditioning set.

\subsection{TPR and FPR for d-Separation}

In Section~\ref{sec:dsep:test} we analyzed the accuracy of $\SCI_f$, $\SCI_q$, $\CMI_{\Gamma}$ and  $G^2$ for identifying d-separation. In Figure~\ref{fig:indep_comp:tfpr}, we plot the true and false positive rates to the corresponding experiment. In addition, we also provide the results for $\SCI_{\textit{fs}}$ and $\CMI_{\Gamma}$ with $\alpha=0.001$. Since we did not provide the accuracy of \JIC for this experiment in the main body of the paper, we plot the accuracy, true and false positive rates of \JIC in Figure~\ref{fig:jic_details} and analyze those results at the end of this section.

From Figure~\ref{fig:indep_comp:tfpr}, we see that $\SCI_f$ and $\SCI_{\textit{fs}}$ perform best. Only for very high noise setups ($\ge 70 \%$) they start to flag everything as independent. The $G^2$ test struggles with small sample sizes. It needs more than $500$ and is inconsistent given more than $35 \%$ noise.  Note that we forced $G^2$ to decide for every sample size, while the minimum number of samples recommended for $G^2$ on this data set would be $1 \, 440$, which corresponds to $10(|\mathcal{X}|-1)(|\mathcal{Y}|-1)|\mathcal{Z}|$~\citep{kalisch:12:pcalg}. Further, we observe that there is barely any difference between $\CMI_{\Gamma}$ using $\alpha = 0.05$ or $\alpha=0.001$ as a significance level. After more than $20 \%$ noise has been added, $\CMI_{\Gamma}$ starts to flag everything as dependent.

Next, we also show the accuracy for identifying d-separation for $\CMI$ with zero as threshold in Figure~\ref{fig:cmi_details}. Overall, it performs very poorly, which raises from the fact that it barely finds any independence. In addition to the accuracy of $\CMI$, we also plot the average value that $\CMI$ reports for the true positive case ($F \independent T \mid D,E$), where it should be equal to zero. It can be seen that it is dependent on the noise level as well as the sample size. This could explain, why $\SCI_f$ performs best on the d-separation data. Since the noise is uniform, the threshold for $\SCI_f$ is likely to be higher the more noise has been added.

The \JIC test has the opposite problem. For the d-separation scenario that we picked it is too restrictive and falsely detects independencies where the ground truth is dependent, as shown in Figure~\ref{fig:jic_details}. As the discrete version of \JIC is calculated from the empirical entropies and a penalizing term based on the asymptotic formulation of stochastic complexity---i.e.
\[
\JIC(X;Y \mid Z) := \max \{ \hat{I}(X; Y \mid Z) - \frac{(| \mathcal{X} | - 1)(| \mathcal{Y}| - 1) |\mathcal{Z}|}{2n} \log n, 0 \} \, ,
\]
it penalizes quite strongly in our example since $|\mathcal{Z}| = 16$. As \JIC is based on an asymptotic formulation of stochastic complexity, we expect it to perform better given more data.

\begin{figure}[t]%
	\begin{minipage}[t]{.5\linewidth}
	\centering
	\begin{minipage}[t]{.5\linewidth}
	\includegraphics[]{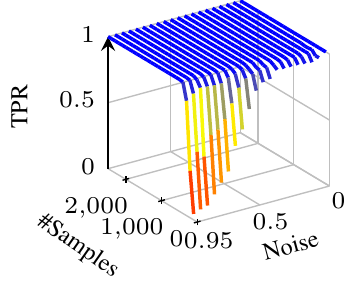}
	\end{minipage}%
	\begin{minipage}[t]{.5\linewidth}
	\includegraphics[]{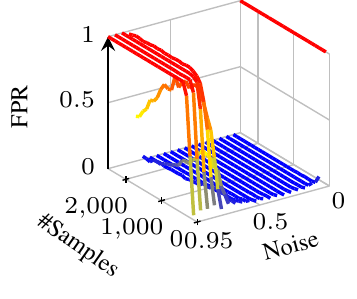}
	\end{minipage}%
	\subcaption{$\SCI_f$}
	\label{fig:tprfpr:scif}
	\end{minipage}
	\begin{minipage}[t]{.5\linewidth}
	\centering
	\begin{minipage}[t]{.5\linewidth}
	\includegraphics[]{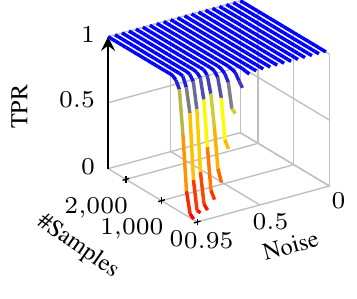}
	\end{minipage}%
	\begin{minipage}[t]{.5\linewidth}
	\includegraphics[]{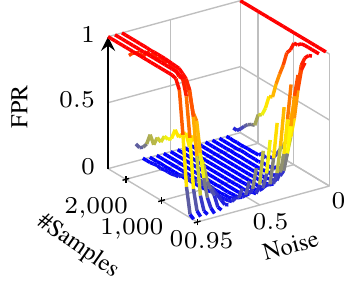}
	\end{minipage}
	\subcaption{$\SCI_q$}
	\label{fig:tprfpr:sciq}
	\end{minipage}%
	\newline
	\begin{minipage}[t]{.5\linewidth}
	\centering
	\begin{minipage}[t]{.5\linewidth}
	\includegraphics[]{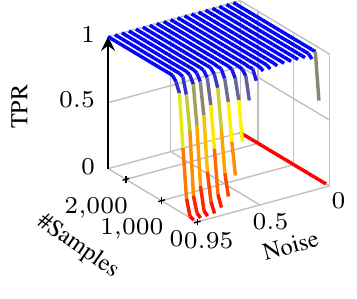}
	\end{minipage}%
	\begin{minipage}[t]{.5\linewidth}
	\includegraphics[]{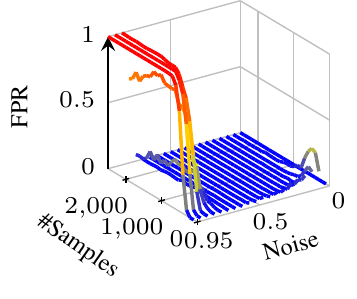}
	\end{minipage}%
	\subcaption{$\SCI_{\textit{fs}}$}
	\label{fig:tprfpr:scifs}
	\end{minipage}%
	\begin{minipage}[t]{.5\linewidth}
	\centering
	\begin{minipage}[t]{.5\linewidth}
	\includegraphics[]{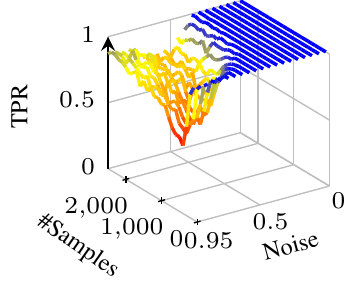}
	\end{minipage}%
	\begin{minipage}[t]{.5\linewidth}
	\includegraphics[]{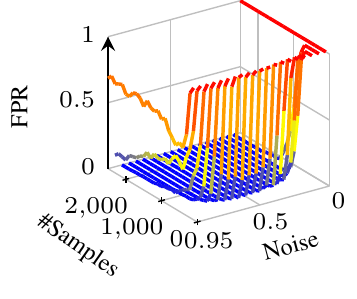}
	\end{minipage}%
	\subcaption{$G^2$}
	\label{fig:tprfpr:g2}
	\end{minipage}%
	\newline
	\begin{minipage}[t]{.5\linewidth}
	\centering
	\begin{minipage}[t]{.5\linewidth}
	\includegraphics[]{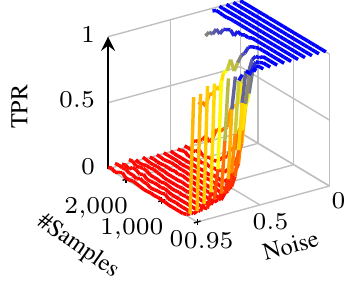}
	\end{minipage}%
	\begin{minipage}[t]{.5\linewidth}
	\includegraphics[]{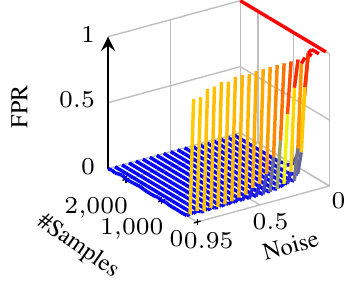}
	\end{minipage}%
	\subcaption{$\Gamma_{.05}$}
	\label{fig:tprfpr:gamma05}
	\end{minipage}%
	\begin{minipage}[t]{.5\linewidth}
	\centering
	\begin{minipage}[t]{.5\linewidth}
	\includegraphics[]{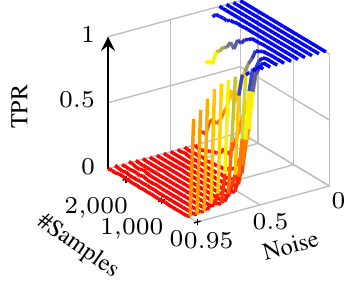}
	\end{minipage}%
	\begin{minipage}[t]{.5\linewidth}
	\includegraphics[]{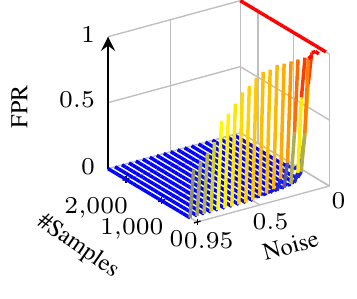}
	\end{minipage}%
	\subcaption{$\Gamma_{.001}$}
	\label{fig:tprfpr:gamma001}
	\end{minipage}%
	\caption{True positive (TPR) and false positive rates (FPR) of $\SCI_f$, $\SCI_q$, $\SCI_{\textit{fs}}$,  $G^2$ and $\CMI_{\Gamma}$ with $\alpha=0.05$ ($\Gamma_{.05}$) and $\alpha = 0.001$ ($\Gamma_{.001}$) for identifying d-separation. We use varying samples sizes and additive noise percentages, where a noise level of $0.95$ refers to $95 \%$ additive noise.} 
	\label{fig:indep_comp:tfpr}
\end{figure}

\begin{figure}[h]%
	\begin{minipage}[t]{.5\linewidth}
	\centering
	\includegraphics[]{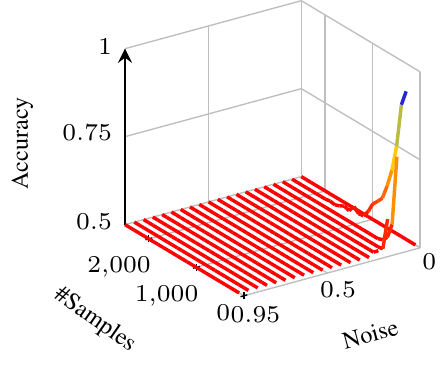}
	\end{minipage}%
	\begin{minipage}[t]{.5\linewidth}
	\centering
	\includegraphics[]{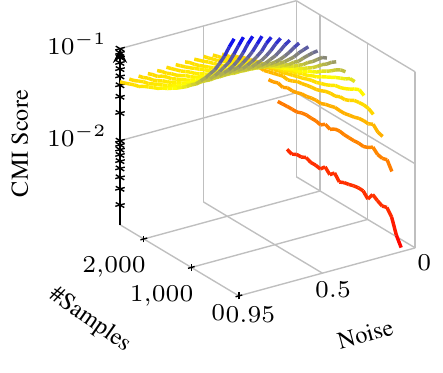}
	\end{minipage}%
	\caption{Accuracy of \CMI (left) and the average value returned by \CMI for the true independent case (right) for varying samples sizes and additive noise percentages. $I(F;T \mid D,E)$ is larger for small sample sizes.} 
	\label{fig:cmi_details}
\end{figure}

\begin{figure}[h]%
	\begin{minipage}[t]{.33\linewidth}
	\centering
	\includegraphics[]{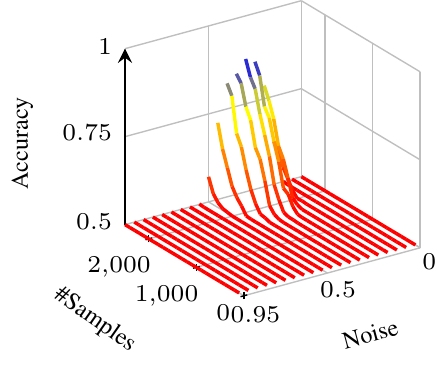}
	\end{minipage}%
	\begin{minipage}[t]{.33\linewidth}
	\centering
	\includegraphics[]{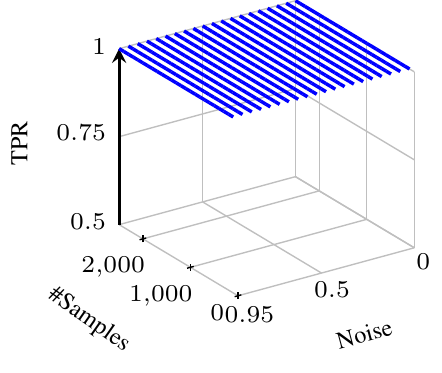}
	\end{minipage}%
	\begin{minipage}[t]{.33\linewidth}
	\centering
	\includegraphics[]{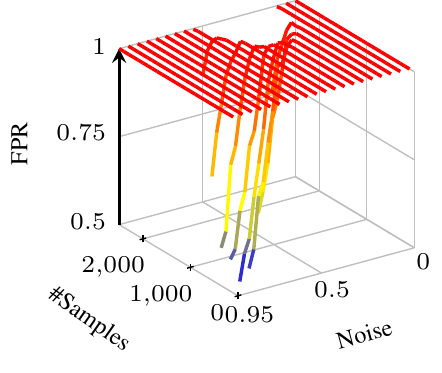}
	\end{minipage}%
	\caption{Accuracy, true positive (TPR) and false positive rates (FPR) of \JIC for identifying d-separation. We use varying samples sizes and additive noise percentages, where a noise level of $0.95$ refers to $95 \%$ additive noise.} 
	\label{fig:jic_details}
\end{figure}

\subsection{Identifying the Parents}

In this experiment, we test the type II error. This we do by generating a certain number of parents $\pa_T$ from which we generate a target node $T$. To generate the parents, we use either a
\begin{itemize}
	\item uniform distribution with a domain size $d$ drawn uniformly with $d \sim \text{unif}(2,5)$,
	\item geometric distribution with parameter $p \sim \text{unif}(0.6,0.8)$,
	\item hyper-geometric distribution with parameter $K \sim \text{unif}(4,6)$, or
	\item poisson distribution with parameter $\lambda \sim \text{unif}(1,2)$.
\end{itemize}
Given the parents, we generate $T$ as a mapping from the Cartesian product of the parents to $T$ plus $10 \%$ additive uniform noise. Then we generate for each distribution $200$ data sets with $2 \; 000$ samples, per number of parents $k \in \{2, \dots, 7 \}$. We apply $\SCI_f$, $\SCI_q$, $\CMI_{\Gamma}$ and $G^2$ on each data set and we check for each $p \in \pa_T$ whether they output the correct result, that is, $p \not \independent T \mid \pa_T \backslash \{ p \}$.

We plot the averaged results for each $k$ in Figure~\ref{fig:parents_identified}. It can clearly be observed that $\SCI_f$ performs best and still has near to $100 \%$ accuracy for seven parents. Although not plotted here, we can add that the competitors struggled most with the data drawn from the poisson distribution. We assume that this is due to the fact that the domain sizes for these data sets were on average larger than for all other tested distributions.

In the next experiment, we generate parents and target in the same way as mentioned above, whereas we now fix the number of parents to three. In addition, we generate $k \in \{ 1, \dots, 7 \}$ random variables $N$ that are drawn jointly independent from $T$ and $\pa_T$ and are uniformly distributed as described above. Then we test whether the conditional independence tests under consideration can still identify for each $p \in \pa_T$ that  $p \not \independent T \mid N \cup \pa_T \backslash \{ p \}$.

The averaged results for $G^2$, $\JIC$, $\SCI_f$, $\SCI_q$ and $\CMI_{\Gamma}$ are plotted in Figure~\ref{fig:parents_identified}. Notice that the results for $G^2$ are barely visible, as they are close to zero for each setup. In general, the trend that we observe is similar to the previous experiment, except that the differences between $\SCI_f$ and its competitors are even larger.

\begin{figure}[h]
	\begin{minipage}[t]{.5\linewidth}
	\centering
	\includegraphics[]{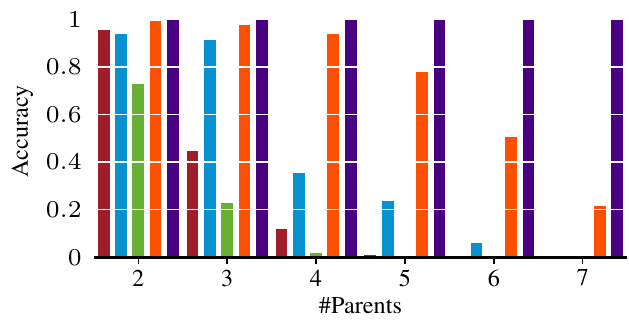}
	\end{minipage}%
	\begin{minipage}[t]{.5\linewidth}
	\centering
	\includegraphics[]{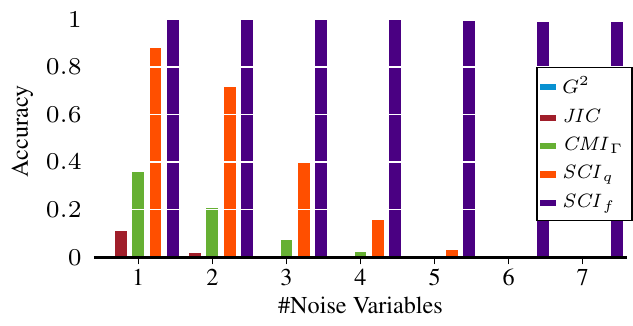}
	\end{minipage}
	\caption{Left: Percentage of parents identified, where we start with only two parents and increase the number of parents to seven. Right: Percentage of parents identified, where we always use three parents, add independently drawn noise variables to the conditioning set.}
	\label{fig:parents_identified}
\end{figure}

\begin{figure}[t]%
	\begin{minipage}[t]{.5\linewidth}
	\centering
	\includegraphics[]{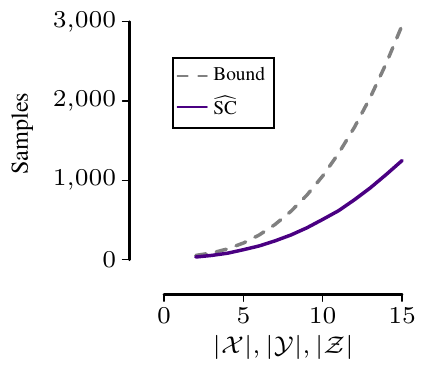}
	\end{minipage}%
	\begin{minipage}[t]{.5\linewidth}
	\centering
	\includegraphics[]{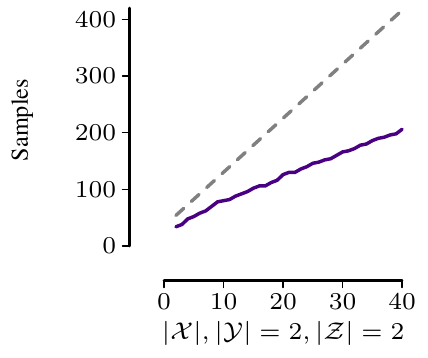}
	\end{minipage}%
	\newline
	\begin{minipage}[t]{.5\linewidth}
	\centering
	\includegraphics[]{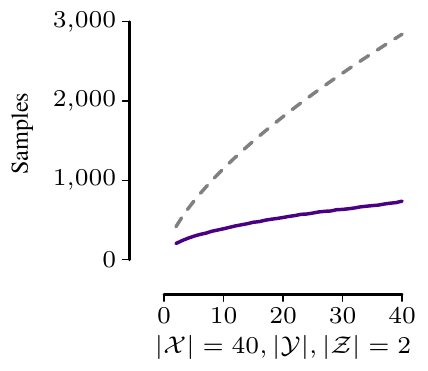}
	\end{minipage}%
	\begin{minipage}[t]{.5\linewidth}
	\centering
	\includegraphics[]{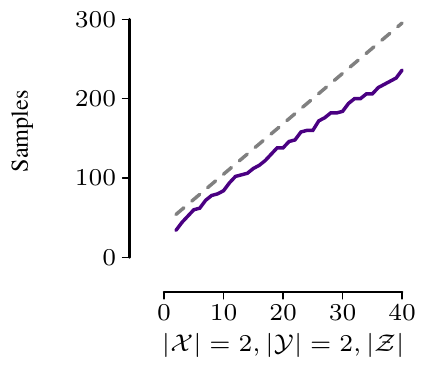}
	\end{minipage}%
	\caption{Estimated sample complexities for independently generated data s.t. $P(|\SCI_f^n / n - I | \ge 0.05) \le 0.05$ and the suggested bound function that is calculated as $35 + 2|\mathcal{X}| |\mathcal{Y}|^{\frac{2}{3}} (|\mathcal{Z}|+1)$. For all setups, increasing the domain size of $X$, $Y$, $Z$ together or independently, the bound function is larger than the empirical value.}
	\label{fig:sample-complexity-in-depth}
\end{figure}

\subsection{Empirical Sample Complexity}

To give an intuition to the sample complexity of \SCI, we provide an empirical evaluation. The goal of this section is to show that there exists a bound for the sample complexity of \SCI, that is sub-linear w.r.t the size Cartesian product of the domain sizes and always larger than the bounds calculated from synthetic data. However, we do not argue that this is the minimal bound that can be found, nor that it is impossible to pass the bound, as we can only evaluate a subset of all possible data sets. What makes us optimistic is that it has been shown that there exists an algorithm with sub-linear sample complexity to estimate \CMI~\citep{canonne:18:sublinear:samples}.

The problem that we would like to solve is to provide a formula that calculates the number of samples $n$ such that $P(|\SCI_f^n(X;Y \mid Z)-I(X;Y \mid Z)| \ge \epsilon) \le \delta$, for small $\epsilon$ and $\delta$. Thereby, we focus on an $n$ such that the probability of making a type I error, i.e. rejecting independence when $H_0 \colon X \independent Y \mid Z$ is true, is low. In our empirical evaluation, we set $\epsilon = \delta = 0.05$ and draw samples from data with the ground truth $I(X;Y \mid Z)=0$ by assigning equal probabilities to each value combination of $X$, $Y$ and $Z$---i.e. we set $P(x,y,z) = \frac{1}{|\mathcal{X}||\mathcal{Y}||\mathcal{Z}|}$ for each value configuration $(x,y,z) \in \mathcal{X} \times \mathcal{Y} \times \mathcal{Z}$. We conduct empirical evaluations for varying domain sizes of $X$, $Y$ and $Z$, where we define w.l.o.g. $|\mathcal{X}| \ge |\mathcal{Y}|$, as the test is symmetric. For each combination of domain sizes, we calculate $P(|\SCI_f^n(X;Y \mid Z)-I(X;Y \mid Z)| \ge \epsilon) = P(\SCI_f^n(X;Y \mid Z) \ge 0.05) \le 0.05$ as follows: We start with a small $n$, e.g. $2$, generate $1 \; 000$ data sets and check if over those data sets $P(\SCI_f^n(X;Y \mid Z) \ge 0.05) \le 0.05$ holds. If not, we increase $n$ by the minimum domain size of $X$, $Y$ and $Z$. We repeat this procedure until we reach an $n$, for which $P(\SCI_f^n(X;Y \mid Z) \ge 0.05) \le 0.05$ holds and report this $n$.

In Figure~\ref{fig:sample-complexity-in-depth} we plot those values for varying either the domain sizes of $X$, $Y$ or $Z$ independently or jointly. From these evaluations, we handcrafted a formula that shows that it is possible to find an $n$ that is sub-linear w.r.t. the domain sizes of $X$, $Y$ and $Z$ for which empirically $P(\SCI_f^n(X;Y \mid Z) \ge 0.05) \le 0.05$ always holds. Hence, we additionally plot for each domain size the corresponding suggested bound for the sample complexity w.r.t. the formula $35 + 2|\mathcal{X}| |\mathcal{Y}|^{\frac{2}{3}} (|\mathcal{Z}|+1)$. We observe that the empirical values for $n$ are always smaller than the values provided by this formula. We want to emphasize that this is only an example function to show the existence of a sub-linear bound for this data. From the plots we would expect that there exists a tighter bound, however, we did not optimize for that. For future work we would like to theoretically validate a sub-linear bound function.

\end{document}